\documentclass[letterpaper, 10pt]{IEEEtran}

\usepackage{cite}
\usepackage{amsmath,amssymb,amsfonts,amsthm}
\usepackage{hyperref}
\usepackage{epstopdf,color}
\usepackage{textcomp}
\usepackage{graphicx,color}
\graphicspath{{img/}}
\usepackage{mathrsfs}
\usepackage{subfigure}
\usepackage{url}
\usepackage{color}
\usepackage{dsfont}
\usepackage{bbm}
\usepackage{booktabs}
\usepackage{array}
\usepackage[table]{xcolor}
\usepackage{yfonts}
\usepackage{cleveref} 
\usepackage{tikz}
\usepackage{pgfplots}
\usepackage{multirow}
\usepackage{textcomp}
\usepackage{tabularx}
\usepackage{placeins}
\usepackage{float}
\usepackage{nccmath}
\usepackage{accents}
\usepackage{multicol}
\usepackage{etoolbox, refcount}
\usepackage{chngcntr}
\usepackage{apptools}
\usepackage{relsize}
\usepackage[font=footnotesize]{caption}
\usepackage{algorithm}
\usepackage{algpseudocode}

\newtheorem{theorem}{Theorem}[section]
\newtheorem{lemma}[theorem]{Lemma}
\newtheorem{definition}[theorem]{Definition}
\newtheorem{corollary}[theorem]{Corollary}
\newtheorem{proposition}[theorem]{Proposition}

\newtheorem{remark}[theorem]{Remark}

\DeclareMathAlphabet{\mymathbb}{U}{BOONDOX-ds}{m}{n}

\newcommand{\setdef}[2]{\{#1 \; : \; #2\}}

\newcommand{\zero}{\mymathbb{0}}

\newcommand{\real}{\mathbb{R}}

\newcommand{\integerspos}{\mathbb{Z}_{> 0}}

\newcommand{\Cc}{\mathcal{C}}
\newcommand{\Dc}{\mathcal{D}}
\newcommand{\Sc}{\mathcal{S}}
\newcommand{\Ec}{\mathcal{E}}
\newcommand{\Kc}{\mathcal{K}}
\newcommand{\Oc}{\mathcal{O}}
\newcommand{\Pc}{\mathcal{P}}
\newcommand{\Jc}{\mathcal{J}}

\newcommand{\Fc}{\mathcal{F}}
\newcommand{\Lc}{\mathcal{L}}
\newcommand{\Ac}{\mathcal{A}}
\newcommand{\Rc}{\mathcal{R}}
\newcommand{\Ic}{\mathcal{I}}
\newcommand{\Tc}{\mathcal{T}}
\newcommand{\Xc}{\mathcal{X}}
\newcommand{\Bc}{\mathcal{B}}

\newcommand{\Mc}{\mathcal{M}}
\newcommand{\Gc}{\mathcal{M}}
\newcommand{\Nc}{\mathcal{N}}
\newcommand{\Zc}{\mathcal{Z}}
\newcommand{\Yc}{\mathcal{Y}}

\newcommand{\argmin}[2] {\mathrm{arg}\min_{#1}#2}

\DeclareSymbolFont{bbold}{U}{bbold}{m}{n}
\DeclareSymbolFontAlphabet{\mathbbold}{bbold}

\newcommand{\norm}[1]{\left\lVert#1\right\rVert}

\newcommand{\normB}[1]{\Big\lVert#1\Big\rVert}

\newcommand{\Cl}{\operatorname{Cl}}

\newcommand\oprocendsymbol{\hbox{$\bullet$}}
\newcommand\oprocend{\relax\ifmmode\else\unskip\hfill\fi\oprocendsymbol}

\renewcommand{\theenumi}{(\roman{enumi})}

\newcommand*{\QEDA}{\hfill\ensuremath{\blacksquare}}%

\newcommand\xqed[1]{%
  \leavevmode\unskip\penalty9999 \hbox{}\nobreak\hfill
  \quad\hbox{#1}}
\newcommand\demo{\xqed{$\bullet$}}

\newcounter{countitems}
\newcounter{nextitemizecount}
\newcommand{\setupcountitems}{%
  \stepcounter{nextitemizecount}%
  \setcounter{countitems}{0}%
  \preto\item{\stepcounter{countitems}}%
}
\makeatletter
\newcommand{\computecountitems}{%
  \edef\@currentlabel{\number\c@countitems}%
  \label{countitems@\number\numexpr\value{nextitemizecount}-1\relax}%
}
\newcommand{\nextitemizecount}{%
  \getrefnumber{countitems@\number\c@nextitemizecount}%
}
\newcommand{\previtemizecount}{%
  \getrefnumber{countitems@\number\numexpr\value{nextitemizecount}-1\relax}%
}
\makeatother    
{\end{itemize}%
\unskip\computecountitems\ifnumcomp{\previtemizecount}{>}{4}{\end{multicols}}{}}

\newcommand{\longthmtitle}[1]{\mbox{}\textit{(#1):}}

\renewcommand{\baselinestretch}{.99}
\newcommand{\comment}[1]{} 
\newcolumntype{P}[1]{>{\centering\arraybackslash}p{#1}}
\allowdisplaybreaks

\begin{document}
\title{\bf Safe and Dynamically-Feasible Motion Planning using Control
  Lyapunov and Barrier Functions}

\author{Pol Mestres \qquad Carlos Nieto-Granda \qquad Jorge Cort{\'e}s
  \thanks{P. Mestres and J. Cort\'es are with the Department of
    Mechanical and Aerospace Engineering, University of California,
    San Diego, \{pomestre,cortes\}@ucsd.edu. Carlos Nieto-Granda is
    with the DEVCOM US Army Research Laboratory (ARL), Adelphi,
    Maryland, carlos.p.nieto2.civ@army.mil.} %
}

\maketitle

%
%

\begin{abstract}
  This paper considers the problem of designing motion planning
  algorithms for control-affine systems that generate collision-free
  paths from an initial to a final destination and can be executed
  using safe and dynamically-feasible controllers.  We introduce the
  \texttt{C-CLF-CBF-RRT} algorithm, which produces paths with such
  properties and leverages rapidly exploring random trees (RRTs),
  control Lyapunov functions (CLFs) and control barrier functions
  (CBFs).  We show that \texttt{C-CLF-CBF-RRT} is computationally
    efficient for linear systems with polytopic and ellipsoidal constraints, and
  establish its probabilistic completeness. We showcase the
  performance of \texttt{C-CLF-CBF-RRT} in different simulation and
  hardware experiments.
\end{abstract}
%
%

%
%
%
%

\section{Introduction}\label{sec:introduction}

Motion planning refers to the problem of computing a collision-free trajectory for a mobile agent 
to go from an initial state to a goal state. Motion planning algorithms are the backbone of many 
robotics applications, but their implementation remains challenging for robots with complex dynamics
and environments with irregular obstacles. Even in scenarios where the robot dynamics and the environment obstacles are known, obtaining motion plans is in general a challenging task.
%
%
Most motion planning algorithms generate high-level 
plans, consisting of sequences of waypoints in the configuration space, and assume the availability of low-level controllers that can follow such waypoints while avoiding collisions with obstacles.
An example of low-level controllers frequently used in applications requiring collision-free navigation
are those based on control barrier functions (CBFs) for safety and control Lyapunov functions (CLFs) for stability.  However, controllers that simultaneously address safety and stability of the different waypoints might in general be not well-defined. This work is motivated by the need to bridge the gap 
between motion planner implementations and low-level CLF-CBF controllers that produce dynamically feasible safe trajectories.


\subsubsection*{Literature Review}

%
%
\textit{Trajectory optimization} methods in motion
  planning~\cite{MD-HGB-HD-PBW:06,FA-APS-RD:12,XZ-AL-FB:21} seek to
  directly design trajectories from the initial state to the goal
  state that take into account the robot dynamics. These methods
  usually formulate the planning problem as a high-dimensional
  nonconvex problem, which can be difficult to solve efficiently by
  off-the-shelf solvers.  To address this,
  it is common to restrict the problem to
  a parametric class:~\cite{JT-JPH:21,JT-JPH:22} uses the so-called
  MINVO basis,~\cite{WD-WG-KW-SS:19} uses B-splines,
  and~\cite{CR-AB-NR:16,NCS-WDC-ADA:23} use polynomial basis. Even
  with the restriction to such parametric classes, the trajectory
  optimization problems remain nonconvex, and their complexity scales
  with the dimension of the parameter space.  One exception is the
  recent work~\cite{TM-MP-DVW-RT:23}, which formulates the trajectory
  optimization problem as a shortest path problem in Graphs of Convex
  Sets~\cite{TM-JU-PP-RT:24}, an optimization framework that allows
  the trajectory optimization problem to be formulated as a
  mixed-integer convex program for trajectories parameterized by
  Bernstein polynomials. Despite the low runtimes that this algorithm
  exhibits in a variety of different robotic systems, it requires a
  partition of the environment in convex sets, which needs to be
  precomputed offline.

Regardless of the computational complexity, the restriction of
  trajectory optimization methods to parametric classes means that
  they are only guaranteed to produce dynamically feasible solutions
  for special classes of systems (for
  example~\cite{JT-JPH:21,JT-JPH:22,WD-WG-KW-SS:19,CR-AB-NR:16} work
  for quadrotors, and~\cite{NCS-WDC-ADA:23} for feedback linearizable
  systems).  Furthermore, since the controllers needed to track these
  trajectories are generally open loop, they do not possess the
  inherent robustness properties associated with feedback control (an
  exception is~\cite{XZ-AL-FB:21}, which uses model predictive control
  to generate optimal trajectories).

Our work is more closely aligned with 
\textit{sampling-based motion planning} methods~\cite{LEK-PS-JCL-MHO:96}, which
seeks to find a collision-free path from an initial state to a goal state through randomly sampling the state space.
Despite its simplicity, it has been shown to be a practical solution for efficiently finding feasible paths even for high-dimensional problems.
Rapidly-exploring random trees (RRTs)~\cite{SML:98} and its variants~\cite{JK-SL:00,SK-EF:11} are 
a family of sampling-based motion planning algorithms that are simple to implement and are probabilistically
complete, meaning that a feasible path (if it exists) is found with probability one as the number of samples goes to infinity.
RRTs build a tree rooted at a starting configuration and efficiently explore the configuration space by adding more samples.
Despite the widespread use of RRT and the variants outlined above, 
their performance in systems with general differential constraints and dynamics remains limited, 
since they rely on the ability to connect any neighboring
nodes of the tree with a dynamically feasible trajectory. 
This requires solving a \textit{two-point boundary value problem} (BVP)~\cite[Chapter 14]{SML:06}, which in general is challenging.
%
%
Different 
works~\cite{DJW-JVDB:13,LY-ZL-KEB:16}
address this problem by developing algorithms that achieve optimality guarantees for different classes of systems
without requiring the use of a BVP solver.  On the one hand,~\cite{DJW-JVDB:13} considers controllable linear systems, for which the explicit solution of the BVP can be computed, and~\cite{LY-ZL-KEB:16} focuses on non-holonomic systems where \textit{Chow's condition} holds, whose accessibility properties can also be used to sidestep the use of a BVP solver.
%
%
Alternatively, other works introduce heuristics that approximate the solution of the BVP:~\cite{EG-RT:10,AP-RP-GK-LK-TLP:12} do it using the linear quadratic regulator, and~\cite{AJL-BS-SML:23} leverages bang-bang controllers.
Other works circumvent solving the BVP by using learning-based
approaches. For instance,~\cite{LP-KOA:15,YL-KEB:11} introduces an
offline machine learning phase that learns the solution of the
BVP,~\cite{WJW-MB-TMM-MW:18} refines the generation of the dataset
used in this offline phase, and~\cite{HTLC-JH-MF-LP-AF:19} learns the
solution of the BVP using reinforcement learning techniques.
%
%
There are also approaches~\cite{RT-IRM-MT-JWR:10} that combine the
  benefits of trajectory optimization methods with RRT, by
  constructing a tree of optimized trajectories along with tubes
  defining their regions of attraction, derived with sum-of-squares
  programming~\cite{PAP:00}.

%
%
Here we bypass the need to solve the BVP or to optimize over sets of trajectories
by using two sets of well-established tools: control Lyapunov functions (CLFs)~\cite{ZA:83}, for designing stabilizing controllers for nonlinear systems, and control barrier functions (CBFs)~\cite{ADA-SC-ME-GN-KS-PT:19,PW-FA:07}, for rendering safe a desired set.
In applications where safety and stability specifications need to be met simultaneously,
the CLF and CBF conditions can be combined in 
a variety of different formulations including a quadratic program with a relaxation variable~\cite{ADA-XX-JWG-PT:17}, safety filters~\cite{LW-ADA-ME:17} (where the CBF condition acts on top of a stabilizing nominal controller), or designs based on penalty methods~\cite{PM-JC:23-csl}. 
Even though these control designs have shown great success in applications such as adaptive cruise control~\cite{ADA-JWG-PT:14} and bipedal walking~\cite{SCH-XX-ADA:15},
different works have shown that, when combined, they can lead to the existence of
undesired equilibria~\cite{MFR-APA-PT:21,XT-DVD:24,YC-PM-EDA-JC:24-cdc}, which can 
even be asymptotically stable and have large regions of attraction, or the lack of feasibility~\cite{WSC-DVD:22-tac,PM-JC:23-csl,PO-JC:19-cdc} 
between the CBF and CLF conditions.

There exist a few works in the
literature~\cite{AA-CB-RT:22,KM-SY-TY-BH-DP-GF:21,GY-BV-ZS-CB-RT:19,AM-QN:21}
that combine the effectiveness of RRT-based algorithms with the safety
guarantees and computational efficiency provided by CBFs and CLFs,
hence also bypassing the need to compute the solution of a
BVP. However, these approaches require the simulation of trajectories
derived from a CLF-CBF-based controller in order to determine whether
new candidate nodes should be added to the tree. The repeated
simulation of such trajectories can significantly slow down the search
for a feasible path and compromise the computational efficiency of the
resulting algorithm.  Moreover, these existing works can be prone
  to safety violations as a consequence of the numerical errors
  introduced when simulating these trajectories, and do not formally
ensure that the low-level CLF and CBF-based controller possesses both
safety and stability
guarantees. Finally,~\cite{GY-MC-AA-AP-RT-CB:23} introduces
  \texttt{LQR-CBF-RRT*}, which is asymptotically optimal and also
  leverages CBFs to ensure collision-free trajectories. Moreover, this
  method does not require simulating trajectories obtained with a
  CLF-CBF-based controller. However, the CBF condition is only
  verified at a finite sequence of points along a trajectory, which
  might compromise safety in-between such sampled points. Furthermore,
  the reference trajectory is generated through an LQR-based
  controller of a linearized model, which might also not be
  stabilizing for the original nonlinear system.
%
%

\subsubsection*{Statement of Contributions}

We consider the problem of designing motion planning algorithms that
generate collision-free paths from an initial to a final destination
for systems with control-affine dynamics.
%
%
To ensure that the sequence of waypoints generated by the
  sampling-based algorithm can be tracked by a controller while
  ensuring safety and stability, we leverage the theory of CBFs and
  CLFs. Our contributions are:
\begin{enumerate}
\item We introduce a result of independent interest which shows that
  the problem of verifying whether a CLF and a CBF are compatible in a
  set of interest can be solved by solving an optimization problem;
\item Although in general such optimization problem is non-convex,
    we show that for linear systems and CBFs of polytopic or
    ellipsoidal obstacles, it reduces to a quadratically constrained
    quadratic program (QCQP), and for CBFs of circular obstacles it
    can be solved in closed form;
  %
  %
\item We leverage the results on compatibility checking of a
    CLF-CBF pair to develop Compatible CLF-CBF-RRT (or
    \texttt{C-CLF-CBF-RRT} for short), a sampling-based motion
    planning algorithm that is a variant of RRT. We show that, by
    construction, \texttt{C-CLF-CBF-RRT} generates collision-free
    paths that can be executed with a CLF-CBF-based controller, and
    formally establish it is probabilistically complete;
  %
  %
\item We show how our proposed approach can be generalized to
    systems where safety constraints have a high relative degree;
\item We illustrate our results in simulation and hardware
    experiments for differential drive robots and compare them with
    the literature, showing that \texttt{C-CLF-CBF-RRT} can generate
    safe and stable paths with a better average execution time.
\end{enumerate}

Noteworthy properties of \texttt{C-CLF-CBF-RRT} as compared to the
  literature are: it does not require generating closed-loop
  trajectories at every sampling step because of the compatibility
  verification of CLF-CBF pairs; it avoids the potential safety
  violations that occur as a consequence of the numerical errors
  introduced when simulating trajectories; its computational
  complexity is tractable provided that the optimization problem
  verifying the compatibility of the CLF and CBF is tractable; and it
  ensures by construction that the sequence of generated waypoints can
  be robustly asymptotically tracked by a safe controller, without
  introducing unwanted dynamical behaviors such as undesired
  equilibria, and while ensuring that the optimization problem
  defining such controller is recursively feasible.
%
%

\subsubsection*{Notation}
We denote by $\mathbb{Z}_{>0}$, $\mathbb{R}$, and $\mathbb{R}_{\geq0}$
the set of positive integers, real, and nonnegative real numbers,
resp.  For $N\in\mathbb{Z}_{>0}$, we denote $[N]=\{1,2,\hdots,N\}$.
Given $x\in\real^n$, $\norm{x}$ denotes its Euclidean norm. Given
matrix $B\in\real^{n\times m}$, $\text{Im}(B)$ denotes its image.  The
symbols $\mathbb{I}_n$, $\zero_n$ denote the identity and zero
matrices of dimension $n\in\mathbb{Z}_{>0}$, and $\textbf{0}_n$ is the
zero vector of dimension $n$.
Given a set $\Sc\subset\real^n$, we denote its boundary by
$\partial \Sc$ and its closure by $\text{Cl}(\Sc)$.  Given an
arbitrary set $A$, 
we let $\Pc(A)$ be the power set of $A$, i.e., the set of all subsets
of $A$, including the empty set and $A$ itself.  We denote by
$\Bc(x,\delta)$ the Euclidean closed ball of center $x\in\real^n$ and
radius $\delta>0$, i.e.,
$\Bc(x,\delta):=\setdef{y\in\real^n}{\norm{y-x}\leq\delta}$.
Given $f:\real^n\to\real^n$, $g:\real^n\to\real^{n\times m}$ and a
smooth $W:\real^n\to\real$, the notation $L_fW:\real^n\to\real$
(resp., $L_gW:\real^n\to\real^m$) denotes the Lie derivative of $W$
with respect to $f$ (resp., $g$), that is $L_fW=\nabla W^T f$ (resp.,
$\nabla W^T g$).  A function $\beta:\real\to\real$ is extended class
$\Kc_{\infty}$ if it is continuous, $\beta(0)=0$, $\beta$ is strictly
increasing and $\lim\limits_{s \to \pm \infty}\beta(s) = \pm\infty$.
A function $V:\real^n\to\real$ is positive definite with respect to
$q\in\real^n$ if $V(q)=0$ and $V(x)>0$ for $x\neq q$.
Given a locally Lipschitz function $f:\real^n\to\real$, its
generalized gradient at $x \in\real^n$
%
%
is
$\partial f( { x} ) = \text{co} \setdef{\lim\limits_{i\to\infty}
  \nabla f(x_i) }{x_i \to { x}, x_i \notin S \cup \Gamma_f}$, where
$\Gamma_f$ is the zero-measure set where $f$ is non-differentiable and
$S$ is any set of measure zero.
%
%
An undirected graph $\Gc$ is a pair $\Gc = (V, \Ec)$, where
$V=\{1, \hdots, N \}$ is a finite set called the vertex set,
$\Ec\subset V\times V$ is called the edge set where $(i,j)\in\Ec$ if
and only if $(j,i)\in\Ec$.  A path in $\Gc$ is a sequence of vertices
$v_1, \hdots, v_k$, with $k\in\mathbb{Z}_{>0}$, such that for all
$i\in[k-1]$, $(v_i, v_{i+1})\in\Ec$.  A tree is an undirected graph in
which there exists a single path between any pair of vertices.

\section{Preliminaries}\label{sec:prelims}

Here we review control Lyapunov functions, control barrier functions,
and rapidly exploring random trees.
For reference, we provide a
summary table of the symbols used throughout the paper in
Table~\ref{tab:1}.

\subsection{Control Lyapunov and Control Barrier
  Functions}\label{subsec:clf-ncbf}
Consider a control-affine system
\begin{align}\label{eq:control-affine-sys}
  \dot{x}=f(x)+g(x)u,
\end{align}
where $f:\real^{n}\to\real^{n}$ and $g:\real^{n}\to\real^{n\times m}$
are locally Lipschitz functions, with $x\in\real^{n}$ the state and
$u\in\real^{m}$ the input. Throughout the paper, and without loss of
generality, we assume $f(0)=0$, so that the origin $x=0$ is the
desired equilibrium point of the (unforced) system.

We start by recalling the notion of Control Lyapunov function
(CLF)~\cite{EDS:98,RAF-PVK:96a}.
\begin{definition}\longthmtitle{Control Lyapunov Function}\label{def:clf}
  Given an open set $\mathcal{D}\subseteq\real^{n}$, a point $q\in\real^n$ with
  $q\in\mathcal{D}$, a continuously differentiable function
  $V:\real^{n}\to\real$ is a \textbf{CLF} with respect to $q$ in $\mathcal{D}$ for the
  system~\eqref{eq:control-affine-sys} if
  \begin{itemize}
  \item $V$ is proper in $\mathcal{D}$, i.e.,
    $\setdef{x\in\mathcal{D}}{V(x)\leq c}$ is a compact set for all
    $c>0$,
  \item $V$ is positive definite with respect to $q$,
  \item there exists a continuous positive definite function $W:\real^{n}\to\real$ with respect to $q$ such that, for each $x\in\mathcal{D}$,
    %
    %
    there exists a control
    $u\in\real^{m}$ satisfying
    \begin{align}\label{eq:clf-ineq}
      L_fV(x)+L_gV(x)u \leq -W(x).
  \end{align}
\end{itemize}
\end{definition}
\medskip

CLFs provide a way to guarantee asymptotic stability of the
origin. Namely, if a Lipschitz controller $u_{\text{st}}:\real^n\to\real^m$ is such that, for every $x\in\mathcal{D}$, $u=u_{\text{st}}(x)$
satisfies~\eqref{eq:clf-ineq}, then the origin is asymptotically stable for the closed-loop system~\cite{EDS:98}.  
%
Such controllers can be
synthesized by means of the pointwise minimum-norm (PMN) control
optimization~\cite[Chapter 4.2]{RAF-PVK:96a},
\begin{align*}
  u(x)
  & = \argmin{u \in\real^{m}}{\frac{1}{2}\norm{u}^2}
  \\
  \notag
    & \qquad \text{s.t.~\eqref{eq:clf-ineq} holds}.
\end{align*}
Note that, at each $x\in \real^n$, this is a quadratic program in~$u$.

Next we define the notion of Boolean Nonsmooth Control Barrier Function (BNCBF), adapted from~\cite[Definition II.8]{PG-JC-ME:21-tac}.
%
%
\begin{definition}\longthmtitle{BNCBF~\cite[Definition II.8]{PG-JC-ME:21-tac}}\label{def:valid-bncbf}
    Given $N\in\mathbb{Z}_{>0}$, let $h_i:\real^n\to\real$, for $i\in[N]$, be continuously differentiable functions.
    Let $h(x)=\max_{i\in[N]}h_i(x)$ and
    \begin{subequations}
    \begin{align}
      \Cc &= \setdef{x\in\real^n}{h(x)\geq0}, \\
      \partial \Cc &= \setdef{x\in\real^n}{h(x)=0}.
    \end{align}
    \label{eq:safe-set}
    \end{subequations}
    Suppose that the set $\Cc$ is nonempty. Then, $h$ is a BNCBF of $\Cc$ for~\eqref{eq:control-affine-sys}
  if there exists a locally Lipschitz extended class $\Kc_{\infty}$ function $\alpha:\real\to\real$ 
  such that for every $x\in\Cc$ there exists $u\in\real^m$
  such that,
  \begin{align*}
    \min_{v\in\partial h(x)} v^T (f(x)+g(x)u) \geq -\alpha(h(x)).
  \end{align*}
\end{definition}
\smallskip
%
%

In case $N=1$, Definition~\ref{def:valid-bncbf} reduces to the standard notion of Control Barrier Function~\cite[Definition 2]{ADA-SC-ME-GN-KS-PT:19}. Given~$x \in \real^n$, let $\Ic(x):=\setdef{i\in[N]}{h(x)=h_i(x)}$ denote the set of \emph{active} functions. The following result is adapted from~\cite[Theorem III.6]{PG-JC-ME:21-tac} 
and provides a sufficient condition for $h$ to be a BNCBF.

\begin{proposition}\longthmtitle{Sufficient Condition for BNCBF}\label{prop:suff-cond-bncbf}
  Suppose there is an extended class $\Kc_{\infty}$ function $\alpha:\real\to\real$ such that, for all $x\in\real^n$, there exists $u\in\real^m$ with
  \begin{align}\label{eq:bncbf-suff-cond}
    L_f h_i(x) + L_g h_i(x)u \geq -\alpha(h(x)),
  \end{align}
  for all $i\in\Ic(x)$. Then, $h$ is a BNCBF of $\Cc$.
\end{proposition}
\smallskip

If a measurable and locally bounded controller $u_s:\real^n\to\real^m$
is such that, for every $x\in\real^n$, $u = u_s(x)$ satisfies~\eqref{eq:bncbf-suff-cond}, then $u_s$ renders $\Cc$ forward invariant (cf.~\cite[Theorem II.7, Definition II.8]{PG-JC-ME:21-tac}).

When dealing with both safety and stability specifications, it is important to note 
that an input $u$ might satisfy~\eqref{eq:clf-ineq} but not~\eqref{eq:bncbf-suff-cond}, or vice versa. 
The following notion, adapted from~\cite[Definition 2.3]{PO-JC:19-cdc},
%
%
captures when a CLF $V$ and  a BNCBF $h$ are compatible.

\begin{definition}\longthmtitle{Compatibility of CLF-BNCBF pair}\label{def:clf-bncbf-compatibility}
  Let $\mathcal{D} \subseteq \real^{n}$ be open, $\Cc \subset \Dc$ be
  closed, $V$ a CLF on $\mathcal{D}$ and $h$ a BNCBF of~$\Cc$. Then,
  $V$ and $h$ are compatible in a set $\tilde{\Dc}\subset\Dc$ if there
  exist a positive definite function $W:\real^n\to\real$ and an
  extended class $\Kc_{\infty}$ function $\alpha:\real\to\real$ such
  that, for all $x\in\tilde{\Dc}$, there exists $u\in\real^{m}$
  satisfying~\eqref{eq:clf-ineq} and~\eqref{eq:bncbf-suff-cond} for
  all $i\in\Ic(x)$ simultaneously.
\end{definition}
\smallskip

If $V$ and $h$ are compatible in a set $\tilde{\Dc}$, we can define the minimum norm controller that satisfies the CLF and BNCBF conditions $u^*:\tilde{\Dc}\to\real^m$
%
%
as follows:
\begin{align}\label{eq:clf-bncbf-controller}
  u^*(x)
  & := \argmin{u \in\real^{m}}{\frac{1}{2}\norm{u}^2}
  \\
  \notag
    & \text{s.t.} \ L_fV(x) + L_gV(x)u \leq -W(x), \\
  \notag
    & \quad \ L_f h_i(x) + L_g h_i(x)u \geq -\alpha(h(x)), \ \forall \ i\in\Ic(x).
\end{align}
If $u^*$ is locally Lipschitz, then it ensures that $\Cc$ is forward invariant and that the origin is asymptotically stable for the closed-loop system.

\subsection{Rapidly-exploring Random Trees (RRTs)}\label{subsec:rrt}
Here, we review \texttt{GEOM-RRT}~\cite{JK-SL:00}, cf. Algorithm~\ref{alg:geom-rrt}, a version of RRT~\cite{SML:98} upon which we rely later.
%
%
%
%
\begin{algorithm}[htb!]
  \caption{\texttt{GEOM-RRT}}\label{alg:geom-rrt}
  \begin{algorithmic}[1]
      \State \textbf{Parameters}: $x_{\text{init}}, \Xc_{\text{goal}}, k, \eta$
      \State $\Tc$.init($x_{\text{init}}$)
      \For{$i\in [1,\hdots,k]$}
        \qquad \State $x_{\text{rand}} \leftarrow$ \texttt{RANDOM}$\underline{\hspace{0.2cm}}$\texttt{STATE}
        \qquad \State $x_{\text{near}} \leftarrow$ \texttt{NEAREST}$\underline{\hspace{0.2cm}}$\texttt{NEIGHBOR}($x_{\text{rand}},\Tc$)
        \qquad \State $x_{\text{new}} \leftarrow$ \texttt{NEW}$\underline{\hspace{0.2cm}}$\texttt{STATE}($x_{\text{rand}}, x_{\text{near}}, \eta$)
        \If{\texttt{COLLISION}$\underline{\hspace{0.2cm}}$\texttt{FREE}($x_{\text{near}},x_{\text{new}}$)} \\
          \qquad \quad $\Tc$.\texttt{add}$\underline{\hspace{0.2cm}}$\texttt{vertex}($x_{\text{new}}$) \\
          \qquad \quad $\Tc$.\texttt{add}$\underline{\hspace{0.2cm}}$\texttt{edge}($x_{\text{near}},x_{\text{new}}$)
          \If{$x_{\text{new}}\in\Xc_{\text{goal}}$} \\
          \qquad \quad \quad \quad \Return $\Tc$
          \EndIf
        \EndIf
      \EndFor
      \State \Return $\Tc$
\end{algorithmic}
\end{algorithm}
The input for \texttt{GEOM-RRT} consists of a state space $\Xc$, an initial configuration $x_{\text{init}}$, goal region
$\Xc_{\text{goal}}$, number of iterations $k$, and a steering parameter $\eta$ whose use is defined
in the sequel.
The algorithm builds a tree $\Tc$ by executing $k$ iterations of the following form:
\begin{quote}
At each iteration, a new random sample $x_{\text{rand}}$ is obtained by uniformly sampling~$\Xc$ using  \texttt{RANDOM}$\underline{\hspace{0.2cm}}$\texttt{STATE}().
The function \texttt{NEAREST}$\underline{\hspace{0.2cm}}$\texttt{NEIGHBOR}($x_{\text{rand}},\Tc$)
returns the vertex $x_{\text{near}}$ from $\Tc$ that is closest in the Euclidean distance to $x_{\text{rand}}$.
Next, a new configuration $x_{\text{new}}\in\Xc$ is returned by the \texttt{NEW}$\underline{\hspace{0.2cm}}$\texttt{STATE}
function such that $x_{\text{new}}$ is on the line segment between $x_{\text{near}}$ and $x_{\text{rand}}$
and the distance $\norm{x_{\text{near}}-x_{\text{new}}}$ is at most $\eta$.
Finally, the function \texttt{COLLISION}$\underline{\hspace{0.2cm}}$\texttt{FREE}($x_{\text{near}},x_{\text{new}}$)
checks whether the straight line
%
%
from $x_{\text{near}}$ and $x_{\text{new}}$ is collision free.
If this is the case, $x_{\text{new}}$ is added as a vertex to $\Tc$ and is connected by an edge from $x_{\text{near}}$.
If $x_{\text{new}}\in\Xc_{\text{goal}}$,
there exists a single path in $\Tc$ from $x_{\text{init}}$ to $x_{\text{new}}$.    
\end{quote}

A notable property of \texttt{GEOM-RRT} is that it is \textit{probabilistically complete}, meaning that the probability that the algorithm will return a collision-free path from the initial state to the goal state
%
%
(if one exists) approaches one as the number of iterations tends to infinity~\cite{MK-KS-ZL-KB-DH:19}.

\section{Problem Statement}\label{sec:problem-statement}

Let $\Rc$ be a compact and convex set in $\real^n$
containing $M$ known obstacles $\{ \Oc_l \}_{l=1}^M$,
with $\text{Int}(\Oc_i)\cap\text{Int}(\Oc_j)=\emptyset$ for all $i \neq j\in[M]$. Let 
$\Fc:=\Rc \backslash \cup_{l=1}^M \Oc_l$ denote the \textit{safe} space.
For each $l\in[M]$, we assume that there exists a positive integer
$N_l\in\mathbb{Z}_{>0}$ and known
continuously differentiable functions $\{h_{i,l}:\real^n\to\real\}_{i \in [N_l]}$ such that
$\Oc_l := \setdef{x\in\real^n}{ h_l(x) = \max_{i\in[N_l]} h_{i,l}(x) < 0 }$.
%
%
Even though this
imposes a specific structure on the set $\Oc_l$, one can obtain more complex obstacles by considering sets of the form $\cup_{i\in\mathcal{M}}\Oc_i$, with
$\mathcal{M}$ a subset of~$[M]$.
%
%

The robot dynamics are control-affine of the
form~\eqref{eq:control-affine-sys}, with $f:\real^n\to\real^n$ and
$g:\real^n\to\real^m$ locally Lipschitz. For each $l\in[M]$, $h_l$ is
a BNCBF of $\real^n\backslash\Oc_l$ for these dynamics, with
associated extended class $\Kc_{\infty}$ function $\alpha_l$. We also
assume
\begin{align*}
\nabla h_{i,l}(x)^T g(x) \neq \textbf{0}_{m}, \quad \forall x\in\Fc, \, 
l\in[M], \, i\in[N_l], 
\end{align*}
i.e., one differentiation of $h_{i,l}$ already makes the input $u$ appear explicitly. 
%
%
Given an initial state $x_{\text{init}}\in\Rc$ and a final goal set $\Xc_{\text{goal}}\subset\Rc$, 
our aim is to develop a
sampling-based motion planning algorithm that constructs a collision-free
path $\Ac:=\{ x_i \}_{i=1}^{N_a}$ from $x_{\text{init}}$ to $\Xc_{\text{goal}}$ 
that is dynamically feasible, i.e., such that for each pair of consecutive waypoints in~$\Ac$, there exists a control law that generates a safe trajectory that connects them. Our approach to solve this problem leverages the theory of CLFs and BNCBFs to design controllers which (i) have safety and stability guarantees by design, and (ii) can be implemented efficiently to help reduce the computational burden of generating dynamically feasible trajectories.
%
%

\section{CLF and BNCBF Compatibility Verification}\label{sec:compat-verification}

The key challenge in our proposed approach to the problem outlined in Section~\ref{sec:problem-statement}
is that the optimization~\eqref{eq:clf-bncbf-controller} defining the
CLF-CBF-based controller has to be feasible at all points along the trajectory.
In this section we tackle this problem and show how such a feasibility check can be performed  in general, and how it is efficient in two specific cases of interest.
%
%

\subsection{Compatibility Verification for General Dynamics and
  Obstacles}\label{sec:compat-verification-general-dynamics-and-obstacles}
In this section we consider the problem of verifying that a CLF and a BNCBF are compatible in systems for general dynamics and obstacles. The following result gives a characterization for when a CLF and a BNCBF are compatible in the region~$\Rc$.
%
%
%
%
\begin{proposition}\longthmtitle{Characterization of CLF-BNCBF Compatibility}\label{prop:clf-bncbf-compat-general}
  Given $q\in\Fc$, let $V_q:\real^n\to\real$ be a CLF
  of~\eqref{eq:control-affine-sys} with respect to $q$.  Let $l\in[M]$
  and assume that $h_l$ is a BNCBF of $\real^n\backslash\Oc_l$.  Let
  $W_q:\real^n\to\real$ be a positive definite function with respect
  to $q$ and $\alpha_l:\real\to\real$ be an extended class
  $\Kc_{\infty}$ function.  For each $\Jc\subset\Pc([N_l])$, let
  $Z_{l,\Jc}:=\setdef{x\in\real^n}{ \Ic_l(x)=\Jc }$ denote the set of
  points where the active constraints defining obstacle $\Oc_l$
  correspond to the indices in~$\Jc$.
  %
  %
  For $\Gamma\subset\Rc$, define
  \begin{subequations}\label{eq:r1-clf-bncbf-compat-check-general}
    \begin{align}
      \zeta_1 &= \min_{ \substack{ x\in\Gamma \\ \{ \beta_i \in \real \}_{i\in\Jc} } }
       \normB{\sum_{i\in\Jc} \beta_i L_g h_{i,l}(x) - L_g V_q(x)}^2 
       \\
      &\qquad \text{s.t.} \quad \beta_i \geq 0, \ i\in \Jc, 
      \\
      &\qquad \qquad \ h_{j,l}(x) \leq h_{i,l}(x), \ \forall j\notin\Jc, i\in\Jc,\label{eq:J-active-r1}
      \\
      &\qquad \qquad \ h_l(x) \geq 0.
    \end{align}
  \end{subequations}
  If $\zeta_1 \neq 0$, then $V_q$ and $h_l$ are compatible in
  $Z_{l,\Jc}\cap\Gamma\cap\Fc$.  Otherwise, if $\zeta_1 = 0$, let
  \begin{subequations}\label{eq:r2-clf-bncbf-compat-check-general}
    \begin{align}
      \zeta_2 &= \min_{ \substack{ x\in\Gamma \\  \{ \beta_i \in \real \}_{i\in\Jc} } }
      \Phi(x,\{ \beta_i \}_{i\in\Jc} ),
      \\
              &\qquad \text{s.t.} \quad \sum_{i\in\Jc} \beta_i L_g
                h_{i,l}(x) = L_g V_q(x),
      \\
              &\qquad \qquad \ \beta_i \geq 0, \ i\in \Jc,
      \\
              &\qquad \qquad \ h_{j,l}(x) \leq h_{i,l}(x), \ \forall j\notin\Jc, i\in\Jc,\label{eq:J-active-r2} \\
              &\qquad \qquad \ h_l(x) \geq 0,
    \end{align}
  \end{subequations}
  for
  $\Phi(x,\{ \beta_i \}_{i\in\Jc})=-W_q(x) - L_f V_q(x) +
  \sum_{i\in\Jc} \beta_i (L_f h_{i,l}(x) + \alpha_l(h_{i,l}(x)))$.  If
  $\zeta_2 \geq 0$, then $V_q$ and $h_l$ are compatible in
  $Z_{l,\Jc}\cap\Gamma\cap\Fc$.
    %
  Conversely, if $V_q$ and $h_l$ are compatible in
  $Z_{l,\Jc}\cap\Gamma\cap\Fc$ then there exists an extended class
  $\Kc_{\infty}$ function $\alpha_l$ and a positive definite function
  $W_q$ with respect to $q$ such that either $\zeta_1\neq0$ or
  $\zeta_1=0$ and $\zeta_2\geq0$.
\end{proposition}
\begin{proof}
%
  %
  First note that if $\zeta_1=0$, the optimization
  problem~\eqref{eq:r2-clf-bncbf-compat-check-general} is feasible and
  therefore $\zeta_2$ is well-defined.  By Farkas'
  Lemma~\cite{RTR:70}, $V_q$ and $h_l$ are compatible at
  $x\in Z_{l,\Jc}\cap\Gamma\cap\Fc$ if and only if for some positive
  definite function $W_q$ with respect to $q$ and some extended
  class $\Kc_{\infty}$ function $\alpha_l$, there do not exist
  $\beta_0\in\real_{\geq0}$,
  $\{ \beta_i \}_{i\in\Jc} \subset \real_{\geq0} $ such that
  \begin{subequations}
    \begin{align}
      &  \beta_0 L_g V_q (x) = \sum_{i\in\Jc} \beta_i L_g h_{i,l}(x),\label{eq:farkas-lemma-general-first}
      \\
      \notag
      &\beta    _0 (-L_f V_q(x) -W(x)) 
      \\   
      &\quad \quad + \sum_{i\in\Jc} \beta_i (\alpha_l(h_{i,l}(x)) + L_f h_{i,l}(x)) < 0.\label{eq:farkas-lemma-general-second}
    \end{align} 
    \label{eq:farkas-lemma-general}
  \end{subequations}
  First suppose that for some $W_q$ and $\alpha_l$, either
  $\zeta_1\neq0$ or $\zeta_1=0$ and $\zeta_2\geq0$.  Suppose there
  exists a solution
  $s_1^*=(x^*, \beta_0^*, \{ \beta_i^* \}_{i\in\Ic_l(x)})$
  of~\eqref{eq:farkas-lemma-general} and let us reach a contradiction.
  If $\beta_0^*=0$, then,~\eqref{eq:farkas-lemma-general} implies that
  the constraints
  $L_f h_{i,l}(x) + L_g h_{i,l}(x)u \geq -\alpha_l(h_{i,l}(x))$ are
  not simultaneously feasible, which means that $h_l$ is not a BNCBF,
  hence arriving at a contradiction.  Therefore, $s_1^*$ must be such
  that $\beta_0^* > 0$.  By taking
  $\tilde{\beta}_i = \frac{\beta_i}{\beta_0}$ for $i\in\Jc$, we deduce
  that $(x^*, \{ \tilde{\beta}_i \}_{i\in\Jc} )$ is a solution
  of~\eqref{eq:r1-clf-bncbf-compat-check-general} with a value of the
  objective function equal to zero. This means that if $\zeta_1\neq0$,
  the solution $s_1^*$ does not exist and $V_q$ and $h_l$ are
  compatible in $Z_{l,\Jc}\cap\Gamma\cap\Fc$.  Otherwise, if
  $\zeta_1=0$, then $(x^*, \{ \tilde{\beta}_i \}_{i\in\Jc} )$ is a
  solution of~\eqref{eq:r2-clf-bncbf-compat-check-general} with a
  strictly negative value of the objective function.  This means that
  if $\zeta_1=0$ and $\zeta_2\geq0$, the solution $s_1^*$ does not
  exist and $V_q$ and $h_l$ are compatible in
  $Z_{l,\Jc}\cap\Gamma\cap\Fc$.  Conversely, suppose that $V_q$ and
  $h_l$ are compatible in $Z_{l,\Jc}\cap\Gamma\cap\Fc$. This implies
  that there exists $W_q$ and $\alpha_l$ such
  that~\eqref{eq:farkas-lemma-general} has no
  solution. If~\eqref{eq:farkas-lemma-general-first} has no solution,
  then $\zeta_1\neq0$. If~\eqref{eq:farkas-lemma-general-first} has a
  solution but~\eqref{eq:farkas-lemma-general-second} does not, then
  $\zeta_1=0$ and $\zeta_2\geq0$.
\end{proof}


%
%
Note that Proposition~\ref{prop:clf-bncbf-compat-general} is valid for
any set $\Gamma\subset\Rc$. Intuitively, since the CLF
  and BNCBF conditions define half-spaces in the control input
  $u$,~\eqref{eq:r1-clf-bncbf-compat-check-general} checks whether the
  normal vectors of the hyperplanes defining such half-spaces are
%
%
linearly independent.
%
%
If this condition does not
hold,~\eqref{eq:r2-clf-bncbf-compat-check-general} checks whether the
input-independent terms of the CLF and BNCBF conditions leave enough
space for such conditions to be compatible.  Additionally,
optimization problems~\eqref{eq:r1-clf-bncbf-compat-check-general}
and~\eqref{eq:r2-clf-bncbf-compat-check-general} need to be checked
for every possible set of active constraints. The
constraints~\eqref{eq:J-active-r1} and~\eqref{eq:J-active-r2} ensure
that $\Jc$ is the set of active constraints at $x$.  Often, one is
interested in verifying the compatibility of a CLF and a BNCBF only in
a small subset of $\Rc$, in which case the flexibility provided by the
set $\Gamma$ is useful.

\begin{remark}\longthmtitle{Checking for all Possible Sets of Active
    Constraints}\label{rem:different-sets-of-active-constraints} 
%
%
  {\rm Given a subset $\Jc\subset\Pc([N_l])$ of functions
    $\{h_{i,l} \}$, Proposition~\ref{prop:clf-bncbf-compat-general}
    provides a way to verify if the CLF and the BNCBF are compatible
    at the points in the region of interest $\Gamma\cap\Fc$ where such
    functions are active. Let
    $H_{l,\Jc}:=\setdef{x\in\Gamma}{\Ic_l(x)=\Jc}$ be the points in
    $\Gamma$ where the constraints with index in $\Jc$ are active, and
    $\Sc_l:=\setdef{\Jc\subset\Pc([N_l])}{ H_{l,\Jc}\neq\emptyset }$
    be the sets of indices for which the above set is nonempty.  The
    class $\Sc_l$ contains all possible sets of active constraints in
    $\Gamma$.  By checking the condition in
    Proposition~\ref{prop:clf-bncbf-compat-general} for all $\Jc$ in
    $\Sc_l$,
  %
  %
    we can verify if the CLF and the BNCBF are compatible in
    $\Gamma\cap\Fc$.  In practice, given a region $\Gamma$ where we
    are interested in checking the compatibility of $V_q$ and $h_l$,
    one can often identify the indices that can achieve a maximum
    value in $\Gamma$ (for example, for polytopic obstacles in the
    plane, only a few of the functions $h_{i,l}$ have points in
    $\Gamma$ where they take positive values). This means that the
    cardinality of $\Sc_l$ is often small and the number of checks
    using Proposition~\ref{prop:clf-bncbf-compat-general} can be kept
    small. \demo }
\end{remark}

\begin{remark}\longthmtitle{Verifying Compatibility for Multiple
    BNCBFs} {\rm Proposition~\ref{prop:clf-bncbf-compat-general}
    actually provides a way to check whether the optimization
    problem~\eqref{eq:clf-bncbf-controller} is feasible at \emph{all}
    points of~$\Gamma$. This can be done as follows: one first finds
    all $l\in[M]$ such that $\Gamma\cap\Oc_l\neq\emptyset$.  If
    $\Gamma$ can be expressed as the $0$-sublevel set of a convex
    differentiable function $\gamma$, i.e.,
    $\Gamma:=\setdef{x\in\real^n}{\gamma(x) \leq 0}$, and the
    functions $h_{i,l}$ are convex, then this can be solved
    efficiently by checking that the solution of the convex problem
    \begin{align*}
      &\min \limits_{x\in\real^n} \gamma(x)
      \\
      &\text{s.t.} \ \ h_{i,l}(x) \leq 0, \quad \forall i\in[N_l]
    \end{align*}
    is non-positive.  The BNCBF constraints associated with those
    $l'\in[M]$ such that $\Gamma\cap\Oc_{l'}=\emptyset$ can be
    neglected since, given a controller that satisfies all the other
    BNCBF constraints, it can be shown to also satisfy the BNCBF
    constraints for such $l'\in[M]$ by taking the corresponding
    extended class $\Kc_{\infty}$ function $\alpha_{l'}$ linear with
    sufficiently large slope.  On the other hand, for $l'\in[M]$ such
    that $\Gamma\cap\Oc_{l'}\neq\emptyset$,
    Proposition~\ref{prop:clf-bncbf-compat-general} ensures that there
    exists a small neighborhood around $\partial\Oc_{l'}$, not
    containing points of any other obstacle, where $V$ and $h_{l'}$
    are compatible. By taking the extended class $\Kc_{\infty}$
    functions of the other CBF constraints as linear functions with
    sufficiently large slope,~\eqref{eq:clf-bncbf-controller} is
    feasible in each of these neighborhoods.  Finally, for points in
    $\Gamma$ not belonging to any of these neighborhoods, the extended
    class $\Kc_{\infty}$ functions can also be taken as linear with
    sufficiently large slope to guarantee
    that~\eqref{eq:clf-bncbf-controller} is feasible.  \demo }
\end{remark}

%
%
\begin{remark}\longthmtitle{About the Choice of CLF and Class
    $\Kc_{\infty}$ Function} {\rm Note that, when solving the
    optimization problems~\eqref{eq:r1-clf-bncbf-compat-check-general}
    and~\eqref{eq:r2-clf-bncbf-compat-check-general} for fixed $V_q$,
    $\alpha_l$, and $W_q$, it is not guaranteed that $\zeta_1\neq0$ or
    $\zeta_1=0$ and $\zeta_2\geq0$.  If $\tilde{\alpha}$ is an
    extended class $\Kc_{\infty}$ function with
    $\tilde{\alpha}(s)\geq\alpha(s)$ for all $s\in\real$, the
    objective function $\Phi$
    of~\eqref{eq:r2-clf-bncbf-compat-check-general} does not decrease
    at any point, which means that the value of $\zeta_1$ remains the
    same, but the condition $\zeta_2 \geq 0$ becomes easier to
    satisfy.
  %
  %
  A similar behavior occurs if $\tilde{W}$ is a positive definite function with $\tilde{W}(x)\leq W(x)$ for
  all $x\in\real^n$. 
  We leverage these observations in Section~\ref{sec:c-clf-cbf-rrt} when we introduce our proposed motion planning algorithm.
  %
  %
  \demo }
\end{remark}

\begin{remark}\longthmtitle{Regularity Properties of the
    Controller}\label{rem:regularity-controller} {\rm If $V_q$ and
    $h_l$ are compatible in $\Rc$ for all $l\in[M]$, the CLF-CBF-based
    controller~\eqref{eq:clf-bncbf-controller} is well defined, i.e.,
    the optimization~\eqref{eq:clf-bncbf-controller} is feasible for
    all points in $\Rc$.  However, slightly stronger conditions are
    needed to ensure that such CLF-CBF-based controller is locally
    Lipschitz and therefore can be used to render $\Cc$ forward
    invariant and the origin asymptotically stable.  We refer the
    reader to~\cite{PM-AA-JC:25-ejc}
  %
  %
  for a survey on different 
  conditions that ensure continuity, Lipschitzness, and other regularity properties of 
  optimization-based controllers of the form~\eqref{eq:clf-bncbf-controller}. These conditions are often satisfied in practice and are mostly related to the dynamics and the specific obstacles, which in our problem here are given and not subject to design. Therefore, throughout this work, we assume that~\eqref{eq:clf-bncbf-controller} satisfies at least one of the sufficient conditions outlined in~\cite{PM-AA-JC:25-ejc} that ensure that the resulting controller is locally Lipschitz.
  \demo }
\end{remark}
%
%
  \begin{remark}\longthmtitle{Input Constraints}\label{rem:input
      constraints} {\rm In many applications, one is interested in
      verifying whether the CLF and BNCBF conditions are
      simultaneously feasible with a control input $u$ constrained to
      lie on the set $\setdef{u\in\real^m}{ C_1 u \leq c_2 }$, with
      $C_1\in\real^{c\times m}$, $c_2\in\real^c$, and
      $c\in\mathbb{Z}_{>0}$. Equivalently, we seek to verify whether
      the inequalities
  \begin{align}\label{eq:input-constrained-inequalities}
    \notag
    &L_f h_{j,l}(x) + L_g h_{j,l}(x) u \geq -\alpha_{i,l}(h_{j,l}(x)), \\
    \notag
    &\qquad \qquad \qquad \qquad \qquad \qquad \forall j\in\Ic_l(x), l\in[M], \\
    &L_f V_i(x) + L_g V_i(x)u + W_i(x) \leq 0, \\
    \notag
    &C_1 u \leq c_2,
  \end{align}
  %
  %
  are simultaneously feasible.  This problem can also be treated using
  Farkas' Lemma~\cite{RTR:70} to obtain a result analogous to
  Proposition~\ref{prop:clf-bncbf-compat-general}.  For example, the
  objective function in~\eqref{eq:r1-clf-bncbf-compat-check-general}
  should be adjusted to
  $\norm{\sum_{i\in\Jc}\beta_j L_g h_{i,l}(x)-L_gV_q(x)-C_1^T
    \bar{\beta} }$, where $\bar{\beta}\in\real^c$ is an additional
  optimization variable with entries that are required to be positive.
  Instead, the objective function
  in~\eqref{eq:r2-clf-bncbf-compat-check-general} should be adjusted
  to
  $\bar{\beta}^T c_2 - W_q(x) - L_f V_q(x) + \sum_{i\in\Jc} \beta_i (
  L_f h_{i,l}(x) + \alpha_l(h_{i,l}(x)) )$.  \demo }
\end{remark}
%
%

Proposition~\ref{prop:clf-bncbf-compat-general} shows that the problem
of checking whether a CLF and a BNCBF are compatible in a region of
interest can be reduced to solving a pair of optimization
problems. However, in general, the optimization
problems~\eqref{eq:r1-clf-bncbf-compat-check-general}
and~\eqref{eq:r2-clf-bncbf-compat-check-general} are not convex and
can be computationally intractable. Our forthcoming exposition
provides two particular cases of dynamics and obstacles for which
these two optimization problems are computationally tractable.

\subsection{Compatibility Verification for Linear Systems and Polytopic Obstacles}\label{sec:linear-systems-polytopic}

In this section we particularize our discussion to linear dynamics,
\begin{align}\label{eq:linear-dynamics}
  \dot{x} = Ax + Bu,
\end{align}
where $A\in\real^{n\times n}$, $B\in\real^{n\times m}$, and the obstacles are polytopic (i.e., the functions $h_{i,l}$ are affine).
We start by introducing some useful notation.
For each $l\in[M]$, let $a_{i,l}\in\real^n$, $b_{i,l}\in\real$ be such that $h_{i,l}(x) = a_{i,l}^T x + b_{i,l}$.
We further assume that $h_l$ is a BNCBF, i.e., there exists an extended class $\Kc_{\infty}$ function
$\alpha_l$ such that,
for all $x\in\real^n\backslash\Oc_l$, there exists $u\in\real^m$ with
\begin{align*}
  a_{i,l}^T (Ax + Bu) \geq -\alpha_l(a_{i,l}^T x + b_{i,l})
\end{align*}
for all $i\in\Ic_l(x)$.
We further assume that given $q\in\real^n$, a quadratic CLF is available, i.e., we have a positive definite 
matrix $P\in\real^{n\times n}$ such that $V_q:\real^n\to\real$, defined as
$V_q(x)=(x-q)^T P (x-q)$, is a CLF with respect to $q$ in $\real^n$ of~\eqref{eq:linear-dynamics}
with associated positive definite function $W_q:\real^n\to\real$.
%
%

The following result follows by applying Proposition~\ref{prop:clf-bncbf-compat-general}
to the case when dynamics are linear and obstacles polytopic.

\begin{proposition}\longthmtitle{CLF-BNCBF Compatibility for Linear Dynamics and Polytopic Obstacles}\label{prop:compat-verification-polytopic}
  Let $\Gamma\subset\Rc$, $l\in[M]$, $\Jc\in\Pc([N_l])$, $q\in\Fc$, and define
  \begin{subequations}\label{eq:first-qcqp}
    \begin{align}
      \zeta_1 &:= \min \limits_{ \substack{ x\in\Gamma \\ \{ \beta_i \in \real \}_{i\in\Jc} } } 
      \norm{ \sum_{i\in\Jc} \beta_i B^T a_{i,l} - B^T P (x-q) }^2 \\
      &\quad \text{s.t.} \ \beta_i \geq 0, \ \forall i\in\Jc, \\
      &\quad \quad \ a_{j,l}^T x + b_{j,l} \leq a_{i,l}^T x + b_{i,l}, \ \forall j\notin \Jc, i\in\Jc, \\
      &\quad \quad \ a_{i,l}^T x + b_{i,l} \geq 0, \ i\in\Jc.
    \end{align}
  \end{subequations}
  If $\zeta_1 \neq 0$, then $V_q$ and $h_l$ are compatible in
  $Z_{l,\Jc} \cap \Gamma \cap \Fc$. Otherwise, if $\zeta_1=0$, let
  \begin{subequations}
  \begin{align}
    \zeta_2 &:= \min \limits_{ \substack{ x\in\Gamma \\ \{ \beta_i \in \real \}_{i\in\Jc} } } \Phi(x,\{ \beta_i \}_{i\in\Jc}) \\
    &\quad \text{s.t.} \ \sum_{i\in\Jc} \beta_i B^T a_{i,l} = B^T P (x-q), \\
    &\quad \quad \ \beta_i \geq 0, \ \forall i\in\Jc, \\
    &\quad \quad \ a_{j,l}^T x + b_{j,l} \leq a_{i,l}^T x + b_{i,l}, \ \forall j\notin \Jc, i\in\Jc, \\
    &\quad \quad \ a_{i,l}^T x + b_{i,l} \geq 0, \ i\in\Jc,
  \end{align}
  \label{eq:second-qcqp}
  \end{subequations}
  with $\Phi(x,\{ \beta_i \}_{i\in\Jc})=-W_q(x)-(x-q)^T PA x + \sum_{i\in \Jc} \beta_i ( \alpha_l(a_{i,l}^T x + b_{i,l}) + a_{i,l}^T Ax)$.
  If $\zeta_2\geq0$, then $V_q$ and $h_l$ are compatible in $Z_{l,\Jc} \cap \Gamma \cap \Fc$.
  Conversely,
  %
  %
  if $V_q$ and $h_l$ are compatible in $Z_{l,\Jc}\cap\Gamma\cap\Fc$, then there exists an extended class $\Kc_{\infty}$ function $\alpha_l$ and a positive definite function $W_q$ with respect to $q$ such that either $\zeta_1\neq0$ or $\zeta_1=0$ and $\zeta_2\geq0$.
\end{proposition}

We end this section by discussing the tractability of 
the optimizations~\eqref{eq:first-qcqp} and~\eqref{eq:second-qcqp}.
If $W_q$ is a quadratic function (as it is often the case in
practice), $\alpha(s)=\alpha_0 s$, with $\alpha_0>0$, and $\Gamma$ is
given by a sublevel set of a quadratic function (e.g., if it is the
sublevel set a quadratic CLF $V_q$), then~\eqref{eq:first-qcqp} and
\eqref{eq:second-qcqp} both have quadratic objective functions and
quadratic constraints, i.e., they are quadratically constrained
quadratic programs (QCQPs).  Moreover, if $\Gamma$ is the sublevel set
of a convex quadratic function, then~\eqref{eq:first-qcqp} is a convex
QCQP (whereas in general,~\eqref{eq:second-qcqp} is non-convex).
If instead $\Gamma$ is the sublevel set of a piecewise linear
  function, both~\eqref{eq:first-qcqp} and \eqref{eq:second-qcqp} have
  affine constraints and therefore are quadratic programs
  (QPs). Moreover,~\eqref{eq:first-qcqp} is a convex QP. In either
  case, even if the resulting QCQPs or QPs are non-convex, there exist
  efficient heuristics~\cite{HT:98,JP-SB:17-arxiv} to solve these
  programs.  Finally,
  Proposition~\ref{prop:compat-verification-polytopic} can be applied
  to settings where obstacles are not polytopic by constructing outer
  approximations of them using polytopes and considering the resulting
  union of convex sets.
%
%

  \subsection{Compatibility Verification for Linear Systems and
    Ellipsoidal Obstacles}\label{sec:integrator-circular}

  In this section, we again consider linear
  dynamics~\eqref{eq:linear-dynamics}, but now assume obstacles are
  ellipsoidal, i.e.,
  $\Oc_l = \setdef{x\in\real^n}{r_l^2 > (x-c_l)^T R_l (x-c_l) }$, for
  some positive definite matrix $R_l\in\real^{n\times n}$,
  $c_l\in\real^n$, and $r_l>0$. In this case, we take
  $h_l(x)= - r_l^2 + (x-c_l)^T R_l (x-c_l)$ (which is continuously
  differentiable and therefore $N_l = 1$ for all $l\in[M]$) and
  $V_q(x)=(x-q)^T P (x-q)$, for some positive definite matrix
  $P\in\real^{n\times n}$.  Then the following result follows from
  applying Proposition~\ref{prop:clf-bncbf-compat-general} to the case
  when dynamics are linear and obstacles are ellipsoidal.
\begin{proposition}\longthmtitle{Sufficient Condition for CLF-BNCBF Compatibility for Linear Dynamics and Ellipsoidal Obstacles}\label{prop:clf-bncbf-compat-linear-dynamics-ellipsoidal}
  Let $\Gamma\subset\Rc$, $l\in[M]$, $q\in\Fc$, $\alpha_l>0$, and define
  \begin{subequations}\label{eq:first-qcqp-ellipsoidal}
    \begin{align}
      \zeta_1 &:= \min \limits_{ x\in\Gamma, y\in\real^n, \beta\in\real } 
      \norm{ B^T y - B^T P (x-q) }^2 \\
      &\quad \text{s.t.} \ \beta \geq 0, \ h_l(x) \geq 0, \ y = -2\beta R_l (x-c_l).
    \end{align}
  \end{subequations}
  If $\zeta_1 \neq 0$, then $V_q$ and $h_l$ are compatible in
  $\Gamma \cap \Fc$. Otherwise, if $\zeta_1=0$, let
  \begin{subequations}\label{eq:second-qcqp-ellipsoidal}
  \begin{align}
    \zeta_2 &:= \min \limits_{ x\in\Gamma, y\in\real^n, \beta\in\real  } \Phi(x,y,\beta) \\
    &\quad \text{s.t.} \ B^T y = B^T P (x-q), \\
    &\quad \quad \ \beta \geq 0, \ h_l(x) \geq 0, \ y = -2\beta R_l (x-c_l),
  \end{align}
  \end{subequations}
  with $\Phi(x,y,\beta)=-W_q(x)-(x-q)^T PAx - \beta\alpha_l r_l^2 - \alpha_l(x-c_l)^T \frac{y}{2} + y^T Ax$.
  If $\zeta_2\geq0$, then $V_q$ and $h_l$ are compatible in $\Gamma \cap \Fc$.
\end{proposition}
If $\Gamma$ is the sublevel set of a quadratic function and $W_q$ is quadratic, both~\eqref{eq:first-qcqp-ellipsoidal}
and~\eqref{eq:second-qcqp-ellipsoidal} are QCQPs and can therefore be solved efficiently~\cite{HT:98,JP-SB:17-arxiv}.
Let us next further restrict our attention to 
single-integrator dynamics, i.e., 
\begin{align}\label{eq:single-integrator}
  \dot{x} = u,
\end{align}
and circular obstacles, i.e., $\Oc_l =
\setdef{x\in\real^n}{\norm{x-c_l} < r_l}$ for some
$c_l\in\real^n$ and $r_l>0$. In this case, we take $h_l(x) =
\norm{x-c_l}^2 - r_l^2$,
%
%
$V_q(x) = \norm{x-q}^2$, and $W_q(x) = (x-q)^T Q
(x-q)$, where $Q\in\real^{n\times
  n}$ is a positive definite matrix.  In this case, the optimization
problems in Proposition~\ref{prop:clf-bncbf-compat-general} can be
solved in closed-form.

\begin{proposition}\longthmtitle{Sufficient Condition for CLF-BNCBF Compatibility for Single Integrator Dynamics and Circular Obstacles}\label{prop:single-integrator-circular-obstacles}
  Let $l\in[M]$, $\alpha_l>0$, $x_0\in\real^n\backslash\{ q \}$, $q\in\Fc$,
  $\Gamma:=\setdef{x\in\real^n}{V_q(x)\leq V_q(x_0)}$,
  $B_l := \norm{q-c_l}_Q^2 - 2\alpha_l r_l^2$,
  \begin{align*}
    \beta_{+} := \frac{ \sqrt{B_l^2 + 4\alpha_l^2 r_l^2 (\norm{q-c_l}^2 - r_l^2) } - B_l }{2\alpha_l r_l^2},
  \end{align*}
  and suppose that one of the following holds:
  \begin{itemize}
    \item $\norm{x_0-q} - \norm{c_l-q} > 0$ and $\frac{\norm{x_0-q}}{\norm{x_0-q} - \norm{c_l-q}} > 1 + \frac{\norm{c_l-q}}{r_l}$;
    \item $\norm{x_0-q} - \norm{c_l-q} > 0$, $\frac{\norm{x_0-q}}{\norm{x_0-q} - \norm{c_l-q}} \leq  1 + \frac{\norm{c_l-q}}{r_l}$ and $\beta_{+} \geq 1+\frac{\norm{c_l-q}}{r_l}$;
    \item $\norm{x_0-q} - \norm{c_l-q} \leq 0$.
  \end{itemize}
  Then, $V_q$ and $h_l$ are compatible in $\Gamma\cap\Fc$. 
\end{proposition}
\begin{proof}
  We rely on Proposition~\ref{prop:clf-bncbf-compat-general}. In the
  setting considered
  here,~\eqref{eq:r1-clf-bncbf-compat-check-general} reads as
  \begin{subequations}\label{eq:first-circular}
    \begin{align}
      {\zeta}_1 &:= \min \limits_{ x\in\Gamma, \beta \in \real } 
      \norm{  2\beta (x-c_l) - 2(x-q) }^2 \\
      &\quad \text{s.t.} \ \beta \geq 0, \\
      &\quad \quad \ \norm{x-c_l}^2 - r_l^2 \geq 0.
    \end{align}
  \end{subequations}
  %
  %
  It follows that $\zeta_1 = 0$ if and only if there exists
  $x\in\Gamma$ and $\beta\in\real\backslash\{ 1 \}$ (note that
  $\beta=1$ and $\zeta_1=0$ are not possible because $q\in\Fc$) such
  that $x = \frac{1}{\beta-1}(\beta c_l - q)$, $\beta\geq0$ and
  $\norm{x-c_l}^2 - r_l^2 \geq 0$.  Equivalently, $\zeta_1=0$ if and
  only if there exists $\beta\in\real\backslash\{ 1 \}$ such that
  $\beta \geq 0$, $|\beta-1| \leq \frac{\norm{c_l-q}}{r_l}$
  %
  %
  and $\beta( \norm{x_0-q}-\norm{c_l-q} ) \geq \norm{x_0-q}$.  Note
  that since $q\in\Fc$, $\norm{c_l-q}\geq r_l$, and therefore the
  condition $\beta \geq 1-\frac{\norm{c_l-q}}{r_l}$ trivially holds if
  $\beta\geq0$. Hence, $\zeta_1=0$ if and only if there exists
  $\beta\in\real\backslash\{ 1 \}$ such that $\beta \geq 0$,
  $\beta \leq 1+\frac{\norm{c_l-q}}{r_l}$, and
  $\beta( \norm{x_0-q}-\norm{c_l-q} ) \geq \norm{x_0-q}$.  We
  distinguish two cases: (i) suppose that
  $\norm{x_0-q} - \norm{c_l-q} \leq 0$.  Then, since $x_0\neq q$, it
  follows that $\beta( \norm{x_0-q}-\norm{c_l-q} ) \geq \norm{x_0-q}$
  can not hold.  Therefore, $\zeta_1\neq0$ and $V_q$ and $h_l$ are
  compatible in $\Gamma$; (ii) suppose instead that
  $\norm{x_0-q} - \norm{c_l-q} > 0$. Then, $\zeta_1=0$ if and only if
  $\frac{\norm{x_0-q}}{\norm{x_0-q} - \norm{c_l-q}} \leq 1 +
  \frac{\norm{c_l-q}}{r_l}$.  Consequently, if
  $\frac{\norm{x_0-q}}{\norm{x_0-q} - \norm{c_l-q}} > 1 +
  \frac{\norm{c_l-q}}{r_l}$, then $V_q$ and $h_l$ are compatible
  in~$\Gamma$. Consider then the case when
  $\frac{\norm{x_0-q}}{\norm{x_0-q} - \norm{c_l-q}} \leq 1 +
  \frac{\norm{c_l-q}}{r_l}$ so that $\zeta_1=0$.
  Then,~\eqref{eq:r2-clf-bncbf-compat-check-general} reads
  \begin{subequations}
    \begin{align}
      \zeta_2 &:= \min \limits_{ \beta\in\real\backslash\{1\} } \frac{1}{(\beta-1)^2}\hat{\Phi}(\beta) \\
      &\quad \text{s.t.} \ \frac{\norm{x_0-q}}{\norm{x_0-q} - \norm{c_l-q}} \leq \beta \leq 1 + \frac{\norm{c_l-q}}{r_l},
    \end{align}
    \label{eq:second-circular}
  \end{subequations}
  where
  $\hat{\Phi}(\beta)=\beta(\alpha_l \norm{q-c_l}^2-\alpha_l
  r_l^2(1-\beta)^2 - \beta (q-c_l)^T Q (q-c_l))$.  By computing the
  roots of $\hat{\Phi}(\beta)=0$, it follows that if
  $\beta_{+} \geq 1+\frac{\norm{c_l-q}}{r_l}$, then
  $\hat{\Phi}(\beta) \geq 0$ for all $\beta \in [0,\beta_{+}]$, which
  implies that $\hat{\Phi}(\beta) \geq 0$ for all
  $\beta \in [ \frac{\norm{x_0-q}}{\norm{x_0-q} - \norm{c_l-q}}, 1 +
  \frac{\norm{c_l-q}}{r_l} ]$, from which it follows that
  $\zeta_2\geq0$ and $V_q$ and $h_l$ are compatible in $\Gamma$.
\end{proof}


Proposition~\ref{prop:single-integrator-circular-obstacles} provides a test for compatibility  over a Lyapunov level set that only requires checking a set of algebraic conditions. Therefore, checking the  compatibility of $V_q=\norm{x-q}^2$ and $h_l(x) = \norm{x-c_l}^2-r_l^2$ over a Lyapunov sublevel set for a single integrator system can be done very efficiently.
%
%
\subsection{Compatibility Verification for Higher Relative Degree Systems}\label{sec:compat-verif-high-rel-degree}
%
%
Here we extend the results of Section~\ref{sec:compat-verification-general-dynamics-and-obstacles} to 
a larger class of system dynamics and barrier functions, specifically 
High-Order Control Barrier
Functions (HOCBFs)~\cite{WX-CB:19}. 
Let $h: \real^n \to \real$ be a continuously differentiable function  defining a
safe set of the form~\eqref{eq:safe-set}. Consider the situation where $h$ has to be
differentiated $\bar{m} \in \integerspos$ times along the dynamics \eqref{eq:control-affine-sys} until the control $u$  appears explicitly (this is referred to as $m$ being the relative degree of $h$ under system~\eqref{eq:control-affine-sys}, cf.~\cite{HKK:02}).

This means that, in order to ensure that the value of $h$ remains
positive at all times (i.e., $\Cc$ is positively invariant), we need
to reason with its higher-order derivatives.  To do so, given
differentiable extended class $\Kc_{\infty}$ functions
$\alpha^{(1)}, \alpha^{(2)}, \hdots, \alpha^{(\bar{m}-1)}$, define a
series of functions
$\phi_0,\dots,\phi_{\bar{m}-1}: \mathbb{R}^n \to \mathbb{R}$ as
follows: $\phi_0 = h$ and
\begin{align*}
  \phi_i(x)= L_f \phi_{i-1}(x)+\alpha^{(i)}(\phi_{i-1}(x)) , \quad i \in \{
  1,\hdots,\bar{m}-1 \} .
\end{align*}
We further define sets $\Cc_1, \hdots,\Cc_{\bar{m}}$ as $ \Cc_1 = \Cc$ and
\begin{align*}
  \Cc_i=\{ x \in \real^n : \phi_{i-1}(x) \geq 0 \}, \quad i \in \{
  2,\hdots, \bar{m} \} .
\end{align*}
The function $h$ is a high-order control barrier function (HOCBF) of $\Cc$ if
  one can find differentiable, extended class $\mathcal{K}_{\infty}$ functions $\alpha^{(1)}, \alpha^{(2)}, \hdots, \alpha^{(m)}$ such that, 
  for all $x \in \Cc \cap \Cc_2 \cap \hdots \cap
  \Cc_{\bar{m}}$, there exists $u \in \real^m$ satisfying
  \begin{align}\label{eq:hocbf-condition}
    L_f \phi_{\bar{m}-1}(x)+L_g \phi_{\bar{m}-1}(x) u + \alpha^{(\bar{m})} (\phi_{\bar{m}-1}(x)) \geq 0 .
  \end{align}
If $\bar{m}=1$, this definition corresponds to the notion of CBF. According to~\cite[Theorem 5]{WX-CB:19}, any locally Lipschitz controller that satisfies \eqref{eq:hocbf-condition} at
each
$x \in \mathcal{C}\cap \mathcal{C}_2 \cap \hdots \cap \mathcal{C}_{\bar{m}}$
renders the set
$\mathcal{C} \cap \mathcal{C}_2 \cap \hdots \cap \mathcal{C}_{\bar{m}}$
positively invariant for system~\eqref{eq:control-affine-sys}.

We next give an analogue of Definition~\ref{def:clf-bncbf-compatibility} for HOCBFs.

\begin{definition}\longthmtitle{Compatibility of CLF-HOCBF pair}\label{def:clf-hocbf-compat}
  Let $\Dc\subset\real^n$ be open, $\Cc\subset\Dc$ be closed, $V$ a CLF on $\Dc$ and $h$ a HOCBF of $\Cc$.
  Then, $V$ and $h$ are compatible at $x\in\Cc\cap\Cc_2\cap\hdots\cap\Cc_{\bar{m}}$ if there exists $u\in\real^m$ satisfying
  ~\eqref{eq:clf-ineq} and~\eqref{eq:hocbf-condition} simultaneously. We refer to both functions as compatible
  in a set $\tilde{\Dc}$ if they are compatible at every point in~$\tilde{\Dc}$.
\end{definition}


The following result is an analogue of Proposition~\ref{prop:clf-bncbf-compat-general} for the case when $h$ is a HOCBF. Its proof follows an analogous argument and we omit it for space reasons.

%
\begin{proposition}\longthmtitle{Characterization of CLF-HOCBF
    Compatibility}\label{prop:clf-hocbf-compat-characterization}
  Given $q\in\Fc$, let $V_q:\real^n\to\real$ be a CLF
  of~\eqref{eq:control-affine-sys} with respect to $q$. Let $h$ be a
  HOCBF of $\Cc$ with relative degree $\bar{m}\in\mathbb{Z}_{>0}$.
  Let $W_q:\real^n\to\real$ be a positive definite function with
  respect to $q$ and
  $\alpha^{(1)}, \alpha^{(2)}, \hdots, \alpha^{(\bar{m})}$ be
  differentiable extended class $\Kc_{\infty}$ functions.  For
  $\Gamma\subset\Rc$, let
  \begin{subequations}
  \begin{align}
    \zeta_1 &= \min\limits_{ x\in\Gamma, \beta\in\real } \norm{\beta
              L_g \phi_{\bar{m}-1}(x) - L_g V_q(x)}^2,
    \\ 
    & \quad \quad \text{s.t.} \quad \beta \geq 0, \ \phi_i(x) \geq 0,
      \ i\in[\bar{m}-1]. 
  \end{align}
  \label{eq:r1-clf-hocbf-compatibility}
  \end{subequations}
  If $\zeta_1 \neq 0$, then $V_q$ and $h$ are compatible in
  $\Gamma\cap\Cc\cap\Cc_2\cap\hdots\Cc_{\bar{m}}$. Otherwise, if
  $\zeta_1 = 0$, let 
  \begin{subequations}
  \begin{align}
    \zeta_2 &= \min\limits_{ x\in\Gamma, \beta\in\real }
              \tilde{\Phi}(x,\beta)
    \\
    & \quad \quad \text{s.t.} \quad \beta \geq 0, \ \phi_i(x) \geq 0, \ i\in[\bar{m}-1],
  \end{align}
  \label{eq:r2-clf-hocbf-compatibility}
  \end{subequations}
  where
  $\tilde{\Phi}(x,\beta)=-W_q(x) - L_f V_q(x) + \beta
  (L_f\phi_{\bar{m}-1}(x)+\alpha^{(\bar{m})}(\phi_{\bar{m}-1}(x)) )$.
  If $\zeta_2\geq0$, then $V_q$ and $h$ are compatible in
  $\Gamma\cap\Cc\cap\Cc_2\cap\hdots\Cc_{\bar{m}}$. Conversely, if
  $V_q$ and $h$ are compatible in
  $\Gamma\cap\Cc\cap\Cc_2\cap\hdots\Cc_m$, then there exists a set of
  differentiable extended class $\Kc_{\infty}$ functions
  $\alpha^{(1)}, \alpha^{(2)}, \hdots, \alpha^{(\bar{m})}$ and a
  positive definite function $W_q$ with respect to $q$
  %
  such that
  either $\zeta_1\neq0$ or $\zeta_1=0$ and $\zeta_2\geq0$.
\end{proposition}

To conclude this section, we consider the case of double-integrator
dynamics and circular obstacles.  The double-integrator dynamics are
given by
\begin{align}\label{eq:double-integrator-dynamics}
  \begin{pmatrix}
    \dot{x} \\ \dot{v}
  \end{pmatrix} = 
  \begin{pmatrix}
    \zero_k & \mathbb{I}_k \\
    \zero_k & \zero_k
  \end{pmatrix}
  \begin{pmatrix}
    x \\ v
  \end{pmatrix}
  + \begin{pmatrix}
    \zero_k \\ \mathbb{I}_k
  \end{pmatrix} u,
\end{align}
with $k\in\mathbb{Z}_{>0}$ such that $n = 2k$, states $x\in\real^k$
and $v\in\real^k$, and input $u\in\real^k$.  As pointed out
in~\cite{GY-CB-RT:19}, only states of the form
$(x_f,\textbf{0}_k)\in\real^{n}$ are stabilizable
for~\eqref{eq:double-integrator-dynamics}, and for any
$x_f\in\real^k$, if we let $q = (x_f,\textbf{0}_n)$, then
$V_q:\real^{n}\to\real$ defined as
$V_q(x,v) = \norm{x-x_f}^2 + \norm{v}^2 + (x-x_f)^T v$ is a CLF with
respect to $q$.  Next, consider $h:\real^{n}\to\real$ given by
$h(x,v) = \norm{x-x_c}^2 - r^2$, for some $x_c\in\real^k$ and $r>0$.
The following result shows that for this choice of $V$ and
$h$,~\eqref{eq:r1-clf-hocbf-compatibility}
and~\eqref{eq:r2-clf-hocbf-compatibility} take a tractable form.

\begin{corollary}\longthmtitle{CLF-HOCBF Compatibility for Circular
    Obstacles and Double
    Integrator}\label{prop:clf-hocbf-compat-circular-obstacles} 
  Consider the double integrator
  dynamics~\eqref{eq:double-integrator-dynamics}. Let
  $q=(x_f,\textbf{0}_k)\in\real^{n}$, and let
  $V_q(x,v) = \norm{x-x_f}^2 + \norm{v}^2 + (x-x_f)^T v$ be a CLF with
  respect to $q$, $W_q:\real^{n}\to\real$ a positive definite function
  with respect to $q$, and $h(x,v)=\norm{x-x_c}^2 - r^2$ for some
  $x_c\in\real^k$, $r>0$ a HOCBF.  Let $\alpha_1>0$, $\alpha_2>0$, and
  $\phi_0:\real^{n}\to\real$, $\phi_1:\real^{n}\to\real$ defined as:
  \begin{align*}
      \phi_0(x,v) &= h(x), \\
      \phi_1(x,v) &= 2(x-x_c)^T v + \alpha_1 ( \norm{x-x_c}^2 - r^2 ),
  \end{align*}
  and $\Cc_1=\setdef{(x,v)\in\real^{2n}}{\phi_1(x,v)\geq0}$.
  For $\Gamma\subset\Rc$, let
  \begin{subequations}
  \begin{align}
    \hat{\zeta}_1 &= \min\limits_{ x\in\Gamma, \beta\in\real,
                    \tilde{x}\in\real^k } \norm{2\tilde{x}-2v-(x-x_f)
                    }^2,
    \\ 
    & \quad \quad \text{s.t.} \quad \beta \geq 0, \ \phi_i(x) \geq 0,
      \ i\in\{0,1\},
    \\
    & \quad \quad \qquad \beta(x-x_c) - \tilde{x}\leq 0, \
      \tilde{x}-\beta(x-x_c) \leq 0. 
  \end{align}
  \label{eq:r1-double-int-clf-hocbf-compatibility}
  \end{subequations}
  If $\hat{\zeta}_1\neq0$, then $V_q$ and $h$ are compatible in
  $\Gamma\cap\Cc\cap\Cc_1$.  Otherwise, if $\hat{\zeta}_1 = 0$, let
  \begin{subequations}
    \begin{align}
      \hat{\zeta}_2 &= \min\limits_{ \substack{ (x,v)\in\Gamma, \beta\in\real, \\ \tilde{x}\in\real^k, \tilde{v}\in\real^k } }
      \hat{\Phi}(x,v,\tilde{x},\tilde{v})
      \\
      & \quad \quad \text{s.t.} \quad \quad \beta \geq 0, \ \phi_i(x)
        \geq 0, \ i\in\{0,1\},
      \\
                    & \quad \quad \qquad \quad 2\tilde{x} - 2v + x-x_f \leq 0,
      \\
                    & \quad \quad \qquad \quad -2\tilde{x} + 2v - (x-x_f) \leq 0,
      \\
      & \quad \quad \qquad \quad \beta (x-x_c) - \tilde{x} \leq 0, \
        \tilde{x} - \beta(x-x_c) \leq 0,
      \\
      & \quad \quad \qquad \quad \beta v - \tilde{v} \leq 0, \ -\beta
        v + \tilde{v} \leq 0, 
    \end{align}
    \label{eq:r2-double-int-clf-hocbf-compatibility}
    \end{subequations}
    where
    $\hat{\Phi}(x,v,\tilde{x},\tilde{v}) = 2\tilde{v}^T v + \alpha_1
    \tilde{x}^T v + 2\alpha_2 \tilde{x}^T v + \alpha_2 \alpha_1
    \tilde{x}^T(x-x_c)-\alpha_1\alpha_2 r^2 \beta - 2(x-x_f)^T v
    -\norm{v}^2 - W_q(x,v)$.  If $\hat{\zeta}_2\geq0$, then $V_q$ and
    $h$ are compatible in $\Gamma\cap\Cc\cap\Cc_1$.
\end{corollary}
\begin{proof}
  The result follows from
  Proposition~\ref{prop:clf-hocbf-compat-characterization} and by
  introducing the new variables $\tilde{x} = \beta(x-x_c)$,
  $\tilde{v}=\beta v$.
\end{proof}

Note that~\eqref{eq:r1-double-int-clf-hocbf-compatibility} is a QCQP,
and if $W_q$ is
quadratic,~\eqref{eq:r2-double-int-clf-hocbf-compatibility} is also a
QCQP and can therefore be solved efficiently~\cite{JP-SB:17-arxiv}.

\section{\texttt{C-CLF-CBF-RRT}}\label{sec:c-clf-cbf-rrt}
In this section, we introduce a novel motion planning algorithm,
termed \texttt{Compatible-CLF-CBF-RRT} (\texttt{C-CLF-CBF-RRT}), that
leverages the compatibility results from
Section~\ref{sec:compat-verification} to generate collision-free paths
that can be tracked using CLF-CBF based controllers.

\subsection{CLF-CBF Compatible Paths}\label{sec:clf-cbf-compatible-paths}

We start by defining formally the type of paths that we seek to find
using our motion planning algorithm. Intuitively, a path is
\textit{CLF-CBF compatible} if the CLF-CBF
controller~\eqref{eq:clf-bncbf-controller} successfully connects pairs
of consecutive waypoints in the path.
%
%

\begin{definition}\longthmtitle{CLF-CBF Compatible Path}\label{def:clf-cbf-compatible-path}
  Let $\Ac = \{ x_i \}_{i=1}^{N_a} \subset \Fc$ be a sequence of
  points, with $N_a\in\mathbb{Z}_{>0}$, $x_1=x_{\text{init}}$ and
  $x_{N_a}\in\Xc_{\text{goal}}:=\Bc(x_{\text{goal}},\delta_{\text{goal}})$,
  where $x_{\text{goal}}\in\real^n$ and
  $\delta_{\text{goal}}>0$. $\Ac$ is a \textit{CLF-CBF compatible
    path} if for each $i\in[N_a-1]$,
  \begin{enumerate}
  \item there exists a CLF $V_i:\real^n\to\real_{\geq0}$ with respect
    to $x_{i+1}$ in an open set containing
    $\Gamma_i:=\setdef{x\in\real^n}{V_i(x)\leq V_i(x_i)}$ for
    system~\eqref{eq:control-affine-sys};
  \item there exist extended class $\Kc_{\infty}$ functions
    $\{\alpha_{i,l}:\real\to\real\}_{l\in[M]}$ and positive definite
    functions $W_i:\real^n\to\real_{\geq0}$ with respect to $x_{i+1}$
    such that the optimization problem
    \begin{align}\label{eq:opt-pb}
      &\min_{u\in\real^m} \frac{1}{2}\norm{u}^2
      \\
      \notag
      &\text{s.t.} \ L_f h_{j,l}(x) + L_g h_{j,l}(x) u \geq
        -\alpha_{i,l}(h_{j,l}(x)),
      \\ 
      \notag
      &\qquad \qquad \qquad \qquad \qquad \qquad \forall j\in\Ic_l(x), l\in[M], \\
      \notag
      &\quad \ L_f V_i(x) + L_g V_i(x)u + W_i(x) \leq 0.
    \end{align}
    is feasible for all $x\in\Gamma_i\cap\Fc$.
  \end{enumerate}
  %
  %
\end{definition}
\smallskip

For each $i\in[N_a-1]$, let $u_i^*:\Gamma_i\cap\Fc\to\real^m$ be a
function mapping each $x\in \Gamma_i\cap\Fc$ to the solution
of~\eqref{eq:opt-pb}.
%
%
Under the assumption that $u_i^*$ is locally Lipschitz,
cf. Remark~\ref{rem:regularity-controller}, the feasibility
  of~\eqref{eq:opt-pb} ensures that the solution of the closed-loop
  system $\dot{x} = f(x) + g(x)u_i^*(x)$ with initial condition $x_i$
  (which we denote as $x(\cdot;x_i)$) is collision-free and
  asymptotically converges to $x_{i+1}$.  Indeed,
\begin{enumerate}
\item the satisfaction of the CLF constraint
  $L_f V_i(x) + L_g V_i(x)u + W_i(x) \leq 0$ at time $t\geq0$ ensures
  that $\frac{d}{dt}V(x(t;x_i))<0$, and $x(t;x_i)$ asymptotically
  converges to $x_{i+1}$;
\item the satisfaction of the BNCBF constraint
  $L_f h_{j,l}(x) + L_g h_{j,l}(x) u \geq -\alpha_{i,l}(h_{j,l}(x))$
  for all $j\in\Ic_l(x)$, $l\in[M]$ at time $t\geq0$ ensures that
  $\frac{d}{dt}h_{j,l}(x(t;x_i)) \geq -\alpha_l(h_{j,l}(x(t;x_i)))$
  for all $j\in\Ic_l(x)$, $l\in[M]$, and $x(t;x_i)$ is collision-free.
\end{enumerate}
Because $x_i\in\Gamma_i\cap\Fc$, this ensures that as long as the CLF
and BNCBF constraints are satisfied, $x(t;x_i)\in\Gamma_i\cap\Fc$. In
turn, since the definition of CLF-CBF compatible path ensures
that~\eqref{eq:opt-pb} is feasible for all $x\in\Gamma_i\cap\Fc$, this
implies that the controller $u_i^*(x(t;x_i))$ is well-defined for all
$t\geq0$, and $x(\cdot;x_i)$ is collision-free and asymptotically
converges to $x_{i+1}$. Therefore, CLF-CBF compatible paths
guarantee that the controller obtained by solving~\eqref{eq:opt-pb}
for each waypoint steers an agent obeying the
dynamics~\eqref{eq:control-affine-sys} towards the next waypoint while
remaining collision-free. Even though the convergence to the waypoint
$x_{i+1}$ is only achieved in infinite time, one can execute the
controller $u_i^*$ until the agent is sufficiently close to $x_{i+1}$
and then switch to the next controller~$u_{i+1}^*$. We elaborate more
on this point in Section~\ref{sec:analysis}, where we identify
conditions on the CLF-CBF compatible path under which the controllers
$\{ u_i^* \}_{i=1}^{N_a-1}$ can steer the agent from a neighborhood of
each waypoint to a neighborhood of the next one, hence
  ensuring that~\eqref{eq:opt-pb} is feasible at all times if we
  switch to the next controller~$u_{i+1}^*$ when the agent is
  sufficiently close to~$x_{i+1}$.
%

%
%
\begin{remark}\longthmtitle{Controllability Requirements for CLF-CBF Compatible Paths}\label{rem:controllability-requirements-clf-cbf-compatible-paths}
{\rm
Definition~\ref{def:clf-cbf-compatible-path} requires each of the points in the path $\Ac$ to be
asymptotically stabilizable. 
This condition imposes some structural properties on the class of systems that admit such paths, which we examine next:
\begin{LaTeXdescription}
\item[Same number of inputs and state variables:] In the case when $m=n$ and $g(x)$ is invertible for all $x\in\real^n$, CLF-CBF compatible paths exist because any point $q\in\real^n$ is asymptotically stabilizable. Indeed, in this setting the function $V_q:\real^n\to\real$ defined by $V_q(x)=\frac{1}{2}\norm{x-q}^2$  is a CLF with respect to~$q$; 
\item[Fewer inputs than state variables:] In the case when $m<n$, the set of stabilizable points is limited. For instance,
for linear systems with $f(x)=Ax$ and $g(x)=B$, with $A\in\real^{n\times n}$ and
$B\in\real^{n\times m}$, only the points $q\in\real^n$ such that $Aq\in\text{Im}(B)$ are stabilizable. 
This is not a major restriction in a lot of cases of interest. 
For example, for a double-integrator system,
where $m=k$ and $n = 2k$, with $k\in\mathbb{Z}_{>0}$, and
\begin{align*}
A = \begin{pmatrix}
  \zero_k & \mathbb{I}_k \\
  \zero_k & \zero_k
\end{pmatrix}, 
\quad
B = \begin{pmatrix}
  \zero_k \\
  \mathbb{I}_k
\end{pmatrix},
\end{align*}
this condition restricts the set of stabilizable points to those that have a zero velocity, but arbitrary position, as pointed out in Section~\ref{sec:compat-verif-high-rel-degree}. In general, if $m < n$, there often exists a smooth change of coordinates $\psi:\real^n\to\real^m$ 
that transforms the dynamics into a single integrator in $\real^m$.
In~\cite[Section IV.A]{PG-IB-ME:19} and~\cite{PM-CNG-JC:24-ral}, for instance, 
this is achieved for unicycle dynamics, by taking the transformation 
$\psi(x_1,x_2,\theta)=[x_1 + l_0 \cos(\theta), x_2 + l_0 \sin(\theta)]$ (where $l_0>0$ is a positive design parameter).
Then, for any $q\in\text{Im}(\psi)$, the set $M_q=\setdef{x\in\real^n}{\psi(x)=q}$ can be asymptotically stabilized. 
Therefore, if $m<n$ but such a transformation $\psi$ exists, Definition~\ref{def:clf-cbf-compatible-path}
can be adapted so that the points in $\Ac$ are in sets of the form~$M_q$. \demo
\end{LaTeXdescription}
}
\end{remark}

\subsection{Algorithm Description}
In this section we introduce the \texttt{C-CLF-CBF-RRT} algorithm, which builds upon RRT, cf. Section~\ref{subsec:rrt}, 
and generates CLF-CBF compatible paths. Algorithm~\ref{alg:compat-clf-cbf-rrt} presents the pseudocode description.

\begin{algorithm}
  \caption{ \texttt{C-CLF-CBF-RRT} }
  \label{alg:compat-clf-cbf-rrt}
  \begin{algorithmic}[1]
      \State \textbf{Parameters}: $\Rc$, $x_{\text{init}}$, $\Xc_{\text{goal}}$, $k$, $\tau$, $\eta, \{ h_l, \alpha_l \}_{l=1}^{M}$
      \State $\Tc$.init($x_{\text{init}}$)
      \For{$i\in [1,\hdots,k]$}
        \qquad \State $x_{\text{rand}} \leftarrow$ \texttt{RANDOM}$\underline{\hspace{0.2cm}}$\texttt{STATE}()
        \qquad \State $x_{\text{near}} \leftarrow$ \texttt{NEAREST}$\underline{\hspace{0.2cm}}$\texttt{NEIGHBOR}($x_{\text{rand}},\Tc$)
        \qquad \State $x_{\text{new}} \leftarrow$ 
        \texttt{NEW}$\underline{\hspace{0.2cm}}$\texttt{STATE}($x_{\text{rand}}, x_{\text{near}},\eta$)
        \If{{\text{not} \texttt{FREE}$\underline{\hspace{0.2cm}}$\texttt{SPACE}($x_{\text{new}}$)}} \\
        \qquad \textbf{skip to next iteration}
        \EndIf
        \qquad \State $V, W \leftarrow$ \texttt{FIND}$\underline{\hspace{0.2cm}}$\texttt{CLF}($x_{\text{new}}$)
        \If{\texttt{COMPATIBILITY}($x_{\text{near}},x_{\text{new}},\tau,\{ h_l, \alpha_l \}_{l=1}^M$,$V$,$W$)} \\
          \qquad \quad $\Tc$.\texttt{add}$\underline{\hspace{0.2cm}}$\texttt{vertex}($x_{\text{new}}$) \\
          \qquad \quad $\Tc$.\texttt{add}$\underline{\hspace{0.2cm}}$\texttt{edge}($x_{\text{near}},x_{\text{new}}$)
          \If{$x_{\text{new}}\in\Xc_{\text{goal}}$} \\
          \qquad \quad \quad \quad \Return $\Tc$
          \EndIf
        \EndIf
      \EndFor
      \State \Return $\Tc$
  \end{algorithmic}
\end{algorithm}
%
%

The input for \texttt{C-CLF-CBF-RRT} consists of a compact, convex set $\Rc\subset\real^n$,
an initial configuration $x_{\text{init}}\in\real^n$, a goal
region $\Xc_{\text{goal}}\subset\real^n$, the number of iterations $k\in\mathbb{Z}_{>0}$ of the algorithm, 
the number of iterations $\tau\in\mathbb{Z}_{>0}$
%
%
for the compatibility check,
a set of extended class $\Kc_{\infty}$ functions $\{ \alpha_l \}_{l=1}^M$,
the steering parameter $\eta>0$, and a set of obstacles $\{ \Oc_l \}_{l=1}^M$ defined by functions
$h_l:\real^n\to\real$ for $l\in[M]$. At the beginning, a tree $\Tc$ is initialized with a single
node at $x_{\text{init}}$ and no edges.

The  \texttt{C-CLF-CBF-RRT} algorithm operates similarly to the \texttt{GEOM-RRT} algorithm described in Section~\ref{subsec:rrt}.
\begin{quote}
    At each iteration, steps 4:-6: are the same as in Algorithm~\ref{alg:geom-rrt}. 
    In general, \texttt{RANDOM}$\underline{\hspace{0.2cm}}$\texttt{STATE} samples $\Rc$ uniformly, but if we know that only a subset of the points in $\Rc$ is stabilizable, one can choose to sample uniformly only over such points.
    The functions \texttt{NEAREST}$\underline{\hspace{0.2cm}}$\texttt{NEIGHBOR} and \texttt{NEW}$\underline{\hspace{0.2cm}}$\texttt{STATE} operate identically to how they do in \texttt{GEOM-RRT}.
    We note that, 
    %
    %
   since $\Rc$ is convex, $x_{\text{new}}$ is guaranteed to belong to it. 
    Next, the function \texttt{FREE}$\underline{\hspace{0.2cm}}$\texttt{SPACE} checks whether $x_{\text{new}}\in\Fc$. If $x_{\text{new}}\notin\Fc$, it skips to the next iteration. Otherwise, \texttt{FIND}$\underline{\hspace{0.2cm}}$\texttt{CLF} finds a CLF $V$ 
   and associated positive definite function $W$ with respect to $x_{\text{new}}$. 
   Then, the \texttt{COMPATIBILITY} function checks whether there exists a CLF-CBF based controller that steers the system from $x_{\text{near}}$ to $x_{\text{new}}$.
 If 
   the \texttt{COMPATIBILITY} function returns a value of \texttt{True}, then
   $x_{\text{new}}$ is added as a vertex to $\Tc$ and is connected by an edge from $x_{\text{near}}$. If $x_{\text{new}}\in\Xc_{\text{goal}}$, there exists a single path in $\Tc$ from $x_{\text{init}}$ to $x_{\text{new}}$.
\end{quote}
%
%
In Section~\ref{sec:compat-def}, we discuss in detail the definition
of the function \texttt{COMPATIBILITY}.
  The function \texttt{FIND}$\underline{\hspace{0.2cm}}$\texttt{CLF}
  aims to find a control Lyapunov function, which is a challenging
  problem for general control systems.  Beyond what we noted in
  Remark~\ref{rem:controllability-requirements-clf-cbf-compatible-paths},
  one can use for this a variety of tools, such as sum-of-squares
  techniques~\cite{WT:06,HD-CJ-HZ-AC:24},
  neural networks~\cite{CD-ZQ-SG-CF:21,YCC-NR-SG:19},
  %
  %
  or the learner-falsifier
  framework~\cite{HR-SS:19}.

\begin{remark}\longthmtitle{Sampling in Systems with Fewer Inputs than State Variables}\label{rem:sampling-underactuated}
{\rm
  A requirement for step 7: of Algorithm~\ref{alg:compat-clf-cbf-rrt} 
  to return a value of \texttt{True} is that $x_{\text{new}}$ is stabilizable.
  Since this point is obtained through random sampling, in general this might not be the case.
  However, if we know the set of points that are stabilizable (for instance,
  an $m$-dimensional manifold $\Mc$ in the case of systems with $m<n$ controls, cf. Remark~\ref{rem:controllability-requirements-clf-cbf-compatible-paths}),
  then we can project $x_{\text{new}}$ onto such set.
  \demo
  }
\end{remark}

\subsection{The \texttt{COMPATIBILITY} function}\label{sec:compat-def}
Here we define the operation of the 
\texttt{COMPATIBILITY} function.
%
%
%
%
%
%
%
%
%
%
Given the CLF $V$ and the positive definite function $W$ with respect to $x_{\text{new}}$ found by \texttt{FIND}$\underline{\hspace{0.2cm}}$\texttt{CLF}, it checks whether the
optimization problem
\begin{align}\label{eq:opt-pb-xnear-compat-check}
  &\min_{u\in\real^m} \frac{1}{2}\norm{u}^2, \\
  \notag
  &\quad \text{s.t.} \ L_fh_{j,l}(x) + L_gh_{j,l}(x)u \geq -\alpha_l(h_{j,l}(x)), \\
  \notag
  &\qquad \qquad \qquad \qquad \qquad \qquad \forall j\in\Ic_l(x), l\in[M], \\
  \notag
  &\quad \quad \ L_fV(x) + L_gV(x)u + W(x) \leq 0.
\end{align}
is feasible for all $x\in\Theta\cap\Fc$, 
where $\Theta=\setdef{x\in\real^n}{V(x)\leq V(x_{\text{near}})}$ and $\alpha_l$ is the class $\Kc_{\infty}$ function associated with $h_l$.
%
%

\emph{1. Find obstacles that intersect domain of interest:}
To check whether~\eqref{eq:opt-pb-xnear-compat-check} is feasible, we first find the obstacles that intersect~$\Theta$,
i.e., we find $l\in[M]$ such that $\Cl(\Oc_l)\cap\Theta\neq\emptyset$.
This can be done by solving the following optimization problem for every $l\in[M]$:
\begin{align}\label{eq:opt-pb-check-collision-gamma}
  &\min_{x\in\real^n} V(x) \\
  \notag
  &\quad \text{s.t.} \ h_{i,l}(x) \leq 0, \quad \forall i\in[N_l].
\end{align}
Then, $\Cl(\mathcal{O}_l)\cap\Theta\neq\emptyset$ iff the optimal
value of~\eqref{eq:opt-pb-check-collision-gamma} is smaller than or
equal to $V(x_{\text{near}})$.
Problem~\eqref{eq:opt-pb-check-collision-gamma} is tractable under the
settings considered in Section~\ref{sec:compat-verification}, where
$V$ is quadratic and the constraints are affine (in which
case~\eqref{eq:opt-pb-check-collision-gamma} is a quadratic program)
or ellipsoidal (in which case~\eqref{eq:opt-pb-check-collision-gamma}
is a~QCQP).

\emph{2. Reduce number of constraints and check for compatibility:}
Next, we check the compatibility of the CLF with each of the CBFs
  associated with the obstacles in
  $\Lc:=\setdef{l\in[M]}{ \Theta\cap\Cl(\Oc_l) \neq \emptyset }$
  (Lemma~\ref{lem:reduction-cbf-set} ensures this step retains
  consistency). Then, \texttt{COMPATIBILITY} uses
Proposition~\ref{prop:clf-bncbf-compat-general} for each
$l\in\Lc$. First, for each $l\in\Lc$, it solves the optimization
problem~\eqref{eq:r1-clf-bncbf-compat-check-general} with
$\Gamma=\Theta$ and obtains the value $\zeta_{1,l}$. If
$\zeta_{1,l}=0$, it
solves~\eqref{eq:r2-clf-bncbf-compat-check-general} with
$\Gamma=\Theta$ and obtains the value $\zeta_{2,l}$.  If for all
$l\in\Lc$, the obtained values of $\zeta_{1,l}$ and $\zeta_{2,l}$ are
such that $\zeta_{1,l}\neq0$ or $\zeta_{1,l}=0$ and
$\zeta_{2,l}\geq0$, then $V$ and $h_l$ are compatible in
$\Theta\cap\Fc$ for all $l\in\Lc$ and \texttt{COMPATIBILITY} returns
\texttt{True}.
%
%

\emph{3. If unsuccessful, increase feasibility set and recheck:}
Otherwise, it updates the set of extended class $\Kc_{\infty}$
functions and the function $W$ in a way that increases the feasible
set of~\eqref{eq:opt-pb-xnear-compat-check}, and performs again the
same check about its feasibility.
In every subsequent iteration, we use a new $W$ obtained by
multiplying the previous one by a constant factor $\sigma\in(0,1)$,
and use linear extended class $\Kc_{\infty}$ functions
$\alpha_l(s) = \alpha_{0,l}s$ with the parameter $\alpha_{0,l}$ being
multiplied by a constant factor $\bar{\sigma} > 1$ at every
iteration. With this choice, the objective function $\Phi$
of~\eqref{eq:r2-clf-bncbf-compat-check-general} does not decrease at
any point, which means that the value of $\zeta_1$ remains the same
but the condition $\zeta_2 \geq 0$ becomes easier to satisfy, which
makes it easier for \texttt{COMPATIBILITY} to return a value of
\texttt{True}.
If after $\tau$ of those updates the function still has not returned a
value of \texttt{True}, it returns a value of \texttt{False}.
We can also employ other heuristics to make it even
  easier for \texttt{COMPATIBILITY} to return a value of
  \texttt{True}. For example, instead of using constant factors
  $\sigma$, $\bar{\sigma}$, one can increase such factors at every
  iteration.

  \begin{remark}\longthmtitle{No Loss of Generality in Assuming Linear
      Class $\Kc_{\infty}$ Function}\label{rem:no-loss-of-generality-linear-class-Kinfty}
    \rm{Since the set $\Theta$ is compact (because $V$ is proper), for
      each $l\in[M]$ and $j\in[N_l]$, the function $h_{j,l}$ is
      bounded in $\Theta$,
      %
      %
      i.e., there exists $M_{j,l}>0$ such that
      $h_{j,l}(x) < M_{j,l}$ for all $x\in\Theta$.  Now suppose that
      $V_q$ and $h_l$ are compatible in $\Theta$, i.e., there exists a
      controller $u_{\text{com}}:\real^n\to\real^m$ such that
      \begin{align*}
        &L_f h_{j,l}(x) + L_g h_{j,l}(x) u_{\text{com}}(x) +
          \alpha_l(h_{j,l}(x)) \geq 0, \ \forall j\in\Ic_l(x),
        \\ 
        &L_f V_q(x) + L_g V_q(x) u_{\text{com}}(x) + W(x) \leq 0,
      \end{align*}
      for all $x\in\Theta$.  Note that  there exists $M_{\text{com}} > 0$
      sufficiently large such that $M_{\text{com}} z > \alpha_l(z)$
      for all $z\in[0,M_{j,l}]$. Using that  $h_{j,l}(x) < M_{j,l}$ for all
      $x\in\Theta$,
      %
      %
      we deduce
  \begin{align*}
    &L_f h_{j,l}(x) + L_g h_{j,l}(x) u_{\text{com}}(x) +
      M_{\text{com}} h_{j,l}(x) \geq 0, \ \forall j\in\Ic_l(x),
    \\
    &L_f V_q(x) + L_g V_q(x) u_{\text{com}}(x) + W(x) \leq 0,
  \end{align*}
  for all $x\in\Theta$.  Therefore, $V_q$ and $h_l$ are also
  compatible in $\Theta$ using a linear class $\Kc_{\infty}$ function
  $\alpha(z) = M_{\text{comp}} z$.  Therefore, without loss of
  generality, we can assume that the class $\Kc_{\infty}$ function
  used in the \texttt{COMPATIBILITY} function is linear.  \demo }
\end{remark}



%
%

\section{Analysis of \texttt{C-CLF-CBF-RRT}}\label{sec:analysis}
In this section we establish the probabilistic completeness of
\texttt{C-CLF-CBF-RRT}. We do this by first showing that if
\texttt{C-CLF-CBF-RRT} returns a tree with a vertex in
$\Xc_{\text{goal}}$, then this tree contains a CLF-CBF compatible
path; and then showing that, under suitable conditions,
\texttt{C-CLF-CBF-RRT} in fact returns a tree with a vertex in
$\Xc_{\text{goal}}$ with high probability.

\begin{proposition}\longthmtitle{\texttt{C-CLF-CBF-RRT} and CLF-CBF Compatible Path}\label{prop:c-clf-cbf-rrt-compatible-path}
  Suppose that \texttt{C-CLF-CBF-RRT} returns a tree $\Tc$ that contains a vertex
  $q_{\text{goal}}\in\Xc_{\text{goal}}$. Then, the single path in $\Tc$ from
  $x_{\text{init}}$ to $q_{\text{goal}}$ is CLF-CBF compatible.
\end{proposition}
\begin{proof}
  Let $N_a\in\mathbb{Z}_{>0}$ and $\Ac=\{x_i\}_{i=1}^{N_a}$ be the path obtained from \texttt{C-CLF-CBF-RRT}, 
  with $x_{1}=x_{\text{init}}$ and 
  $x_{N_a}\in\Xc_{\text{goal}}$. First,  \texttt{FREE}$\underline{\hspace{0.2cm}}$\texttt{SPACE}
  ensures that $x_i\in\Fc$ for all $i\in[N_a]$.
  %
  %
  Moreover, \texttt{FIND}$\underline{\hspace{0.2cm}}$\texttt{CLF}
 ensures that, for all $i\in[N_a-1]$, there exists a CLF $V_i$ with respect to
  $x_{i+1}$, and \texttt{COMPATIBILITY} ensures that there exists 
  a set of class $\Kc_{\infty}$ functions
  $\{ \alpha_{i,l} \}_{l=1}^M$ and a positive definite function $W_i$ with respect to $x_{i+1}$ such that the optimization problem~\eqref{eq:opt-pb}
  is feasible for all points in the set $\setdef{x\in\real^n}{V_i(x)\leq V_i(x_i)}\cap\Fc$.
  This ensures that $\Ac$ is CLF-CBF compatible.
\end{proof}

We next show that, under
some extra assumptions, \texttt{C-CLF-CBF-RRT} returns a tree with 
a vertex in $\Xc_{\text{goal}}$ with probability one as the number of iterations $k$ goes to infinity. In doing so, our next result is critical as it provides conditions under which there exist neighborhoods around a CLF-CBF compatible path for which points of two consecutive neighborhoods can be connected with a CLF-CBF-based controller.
%
%

\begin{lemma}\longthmtitle{Compatibility in Neighboring Vertices}\label{lem:compat-neighboring-vertices-general}
    Let $\Ac=\{ x_i \}_{i=1}^{N_a}$, $N_a\in\mathbb{Z}_{>0}$,  be a CLF-CBF compatible path
    such that there exists $\delta_{\text{clear}}>0$ with $\Bc(x_i,\delta_{\text{clear}})\subset\Fc$ for all $i\in\{2,\hdots,N_a\}$.
    Let $\Nc_1 = \{ x_{\text{init}} \}$.
    For each $i\in\{ 2,\hdots,N_a \}$, assume that there exist sets $\Nc_i$, with $x_{i}\in\Nc_i$, and $\hat{\Gamma}_i$, with $\Gamma_i\subset\hat{\Gamma}_i$ (and $\Gamma_i$ defined as in Definition~\ref{def:clf-cbf-compatible-path}), satisfying the following properties:
    \begin{enumerate}
        \item\label{it:prob-comp-second} for each $y\in\Nc_i$, there exists a CLF $V_y:\hat{\Gamma}_i\to\real$ with respect to $y$ in $\hat{\Gamma}_i$ (with associated positive definite function $W_{y}$) and a bounded controller $\hat{u}_y:\hat{\Gamma}_i\to\real^m$ satisfying the corresponding CLF condition in $\hat{\Gamma}_i$;
        \item\label{it:prob-comp-fourth} there exists a bounded controller $u_i^*:\hat{\Gamma}_i\cap\Fc\to\real^m$ that satisfies the constraints in~\eqref{eq:opt-pb} for all points in $\hat{\Gamma}_i$
        and, for each $y\in\Nc_i$,
        \begin{multline}\label{eq:prob-comp-CLF-gradient-condition}
            |(\nabla V_{y}(x)\!-\!\nabla V_i(x))^T(f(x)\!+\!g(x)u_i^*(x))|
            \\
            < W_i(x),
        \end{multline}
        for all $x\in\Zc=\setdef{z\in\Fc}{\exists l\in[M] \ \text{s.t.} \ d(z,\Oc_l)\leq \tfrac{\delta_{\text{clear}}}{2} }$;
        %
        %
        \item\label{it:prob-comp-fifth} for each $y_2\in\Nc_i$ and $y_1\in\Nc_{i-1}$, $\Gamma_{y_1,y_2}:=\setdef{x\in\real^n}{V_{y_2}(x)\leq V_{y_2}(y_1)} \subset \hat{\Gamma}_i$;
        \item\label{it:prob-comp-sixth} whenever $x_{\text{new}}\in\Nc_i$, global solutions to the optimization problems~\eqref{eq:r1-clf-bncbf-compat-check-general} and~\eqref{eq:r2-clf-bncbf-compat-check-general} in \texttt{COMPATIBILITY} 
        %
        %
        are found.
    \end{enumerate}
     %
    %
    Then, for each $i\in\{ 2,\hdots,N_a\}$, $y_2\in\Nc_i$, and $y_1\in\Nc_{i-1}$,
    there exists a set of extended class $\Kc_{\infty}$ functions $\{ \bar{\alpha}_{i,l} \}_{l=1}^M$ 
    and $\bar{\sigma}>0$ (both dependent on $y_1$, $y_2$) such that, by taking $W_{y_2}^{\bar{\sigma}}(x) = \bar{\sigma}W_{y_2}(x)$, it holds 
    that \texttt{COMPATIBILITY}$(y_1,y_2,1,\{ h_l, \bar{\alpha}_{i,l} \}_{l=1}^M,V_{y_2},W_{y_2}^{\bar{\sigma}})$ = \texttt{True}.
\end{lemma}
%
%
\begin{proof}
%
%
  Given $i\in\{ 2,\hdots,N_a\}$, $y_2\in\Nc_i$, and $y_1\in\Nc_{i-1}$, our goal is to show that there exists a set of extended class $\Kc_{\infty}$ functions $\{ \bar{\alpha}_{i,l} \}_{l=1}^M$
  and a sufficiently small $\bar{\sigma}>0$ such that
  \begin{align}\label{eq:opt-pb-y1-y2-general}
    &\min_{u\in\real^m} \frac{1}{2}\norm{u}^2, \\
    \notag
    &\text{s.t.} \ L_fh_{j,l}(x)+L_gh_{j,l}(x)u \geq - \bar{\alpha}_{i,l}(h_{j,l}(x)), \\
    \notag
    & \qquad \qquad \qquad \qquad \qquad \qquad \forall j\in\Ic_l(x), l\in[M], \\
    \notag
    &\quad \ \nabla V_{y_2}(x)^T ( f(x) + g(x)u ) + \bar{\sigma} W_{y_2}(x) \leq 0,
  \end{align}
  is feasible for all $x\in \Gamma_{y_1,y_2}\cap\Fc$. Figure~\ref{fig:probabilistic-completeness-proof-help} provides a visual aid for the argument that follows. The set $\Gamma_{y_1,y_2}$ is depicted in red, the sets $\Nc_i$ in blue, $\Zc$ in light purple, and the obstacles $\{ \Oc_l \}_{l=1}^M$ in green. For convenience, we let $T_{y_1,y_2} = \Gamma_{y_1,y_2}\cap\Zc$ (depicted in dark purple).
 
  %
  %
  \textbf{Feasibility on $(\Gamma_{y_1,y_2}\backslash T_{y_1,y_2})\cap\Fc$}:
  %
  %
  Since $T_{y_1,y_2}$ contains all points that are closer than $\frac{\delta_{\text{clear}}}{2}$ from the boundary, there exists $h_0>0$ such that $h_{j,l}(x)>h_0$ for all $x\in(\Gamma_{y_1,y_2}\backslash T_{y_1,y_2})\cap\Fc$, $l\in[M]$ and $j\in\Ic_l(x)$.
  Therefore, by taking $\alpha_{i,l}^*>0$, with 
  \begin{align*}
    \alpha_{i,l}^* > \frac{\sup\limits_{\substack{x\in{(\Gamma_{y_1,y_2}\backslash T_{y_1,y_2})\cap\Fc}, \\ j\in\Ic_l(x) 
    }} | L_f h_{j,l}(x)+L_gh_{j,l}(x)\hat{u}_{y_2}(x) | }{h_0},
  \end{align*}
  for each $l\in[M]$ (which exists because $\hat{u}_{y_2}$ is bounded on $\hat{\Gamma}_i$ by~\ref{it:prob-comp-second}), it holds that
  \begin{align*}
    &L_f h_{j,l}(x)+L_gh_{j,l}(x)\hat{u}_{y_2}(x) +\alpha_{i,l}^* h_{j,l}(x) \geq 0, \\
    &\qquad \qquad \qquad \qquad \qquad \qquad \qquad \qquad \forall j\in\Ic_l(x), l\in[M], \\
    &\nabla V_{y_2}(x)^T (f(x) + g(x)\hat{u}_{y_2}(x) ) + \sigma W_{y_2}(x) \leq 0,
  \end{align*}
  for all $x\in(\Gamma_{y_1,y_2}\backslash T_{y_1,y_2})\cap\Fc$ and $\sigma\in(0,1)$, where we have used that $\hat{u}_{y_2}$ satisfies the CLF condition for $V_{y_2}$ by~\ref{it:prob-comp-second}.  

 \textbf{Feasibility on $T_{y_1,y_2}$:}
  From (ii), there exists a bounded controller $u_i^*$ satisfying the constraints in~\eqref{eq:opt-pb} for all $x\in\hat{\Gamma}_i$.
  Since $\Gamma_{y_1,y_2}\subset\hat{\Gamma}_i$, cf.~\ref{it:prob-comp-fifth}, $u_i^*$ satisfies the constraints in~\eqref{eq:opt-pb} for all $x\in\Gamma_{y_1,y_2}$.
  Moreover, since~\eqref{eq:prob-comp-CLF-gradient-condition} holds
  for all $x\in\Zc$ (note that this is only possible because $\Bc(x_i,\delta_{\text{clear}}) \subset \Fc$ and therefore $x_i\notin\Zc$, which means that the right-hand side of~\eqref{eq:prob-comp-CLF-gradient-condition} is strictly positive),
  %
  %
  by~\ref{it:prob-comp-fourth}
  it follows that 
  \begin{align*}
      \nabla V_{y_2}(x)^T (f(x) + g(x)u_i^*(x)) < 0,
  \end{align*}
  %
  %
  for all $x\in T_{y_1,y_2}$.
  Since $\Zc$ is compact, this implies that there exists $\bar{\sigma}\in(0,1)$ sufficiently small such that 
  %
  %
  \begin{align*}
    &L_f h_{j,l}(x)+L_gh_{j,l}(x)u_i^*(x) +\alpha_{i,l}(h_{j,l}(x)) \geq 0, \\
    &\qquad \qquad \qquad \qquad \qquad \qquad \qquad \qquad \forall j\in\Ic_l(x), l\in[M], \\
    &\nabla V_{y_2}(x)^T (f(x) + g(x)u_i^*(x) ) + \bar{\sigma} W_{y_2}(x) \leq 0.
  \end{align*}
  for all $x\in T_{y_1,y_2}$.
  
  Hence, by taking $\bar{\alpha}_{i,l}$ as an extended class $\Kc_{\infty}$ function such that $\bar{\alpha}_{i,l}(s) > \max \{ \alpha_{i,l}(s), \alpha_{i,l}^* s \}$ for all $s\geq0$, and $\bar{\sigma}\in(0,1)$ sufficiently small as described above,~\eqref{eq:opt-pb-y1-y2-general} is feasible for all $x\in \Gamma_{y_1,y_2}\cap\Fc$.
  Since \texttt{COMPATIBILITY} finds the global solutions of the optimization problems~\eqref{eq:r1-clf-bncbf-compat-check-general} and~\eqref{eq:r2-clf-bncbf-compat-check-general}, cf.~\ref{it:prob-comp-sixth}, it follows that \texttt{COMPATIBILITY}$(y_1,y_2,1,\{ h_l, \bar{\alpha}_{i,l} \}_{l=1}^M,V_{y_2},W_{y_2}^{\bar{\sigma}})$ = \texttt{True} (note that since~\eqref{eq:opt-pb-y1-y2-general} includes CBF constraints for $l\in[M]$, this argument is valid independently of the set $\Lc$ found by solving~\eqref{eq:opt-pb-check-collision-gamma}).
\end{proof}

\begin{figure}[htb]
  \centering
  \includegraphics[width=0.99\linewidth]{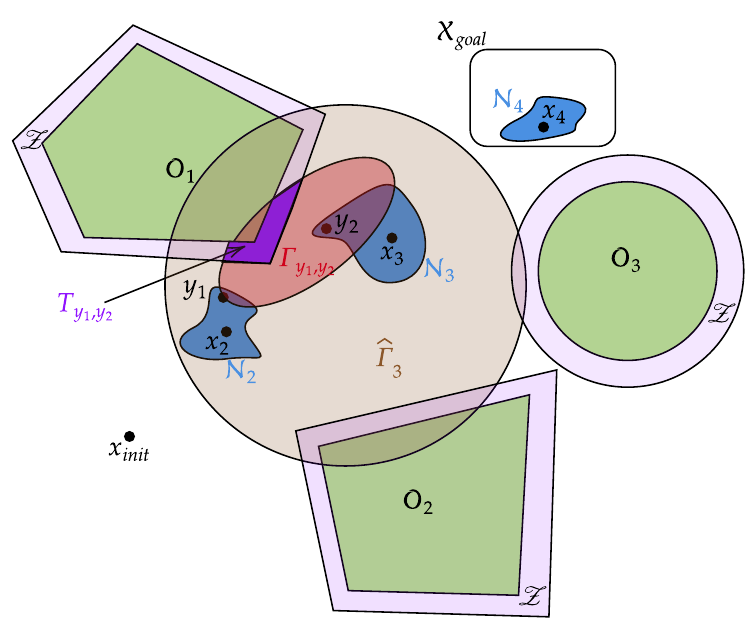}
  \caption{Visual aid for the arguments described in the proof of Lemma~\ref{lem:compat-neighboring-vertices-general}.}
  \label{fig:probabilistic-completeness-proof-help}
\end{figure}

%
%

\begin{remark}\longthmtitle{Verification of Assumptions of Lemma~\ref{lem:compat-neighboring-vertices-general} for Specific Classes of Systems}\label{rem:lemma-prob-com-assumptions-fully-actuated}
{\rm
    For systems with the same number of inputs as state
      variables}, the set $\Nc_i$ in
    Lemma~\ref{lem:compat-neighboring-vertices-general} can be taken
    as a ball centered at the waypoint~$x_i$. 
    As mentioned in Remark~\ref{rem:controllability-requirements-clf-cbf-compatible-paths}, for such  systems, $V_y(x) = \frac{1}{2}\norm{x-y}^2$ is a CLF for any $y\in\real^n$. Moreover, we can take $W_y(x) = \norm{x-y}^2$ and the controller $\hat{u}:\real^n\to\real^n$ defined as 
    $\hat{u}(x) = -\frac{ (x-y_2)^T f(x) + \norm{x-y_2}^2 }{ \norm{g(x)^T (x-y_2)}^2 } g(x)^T (x-y_2)$ is such that
    $(x-y_2)^T (f(x) + g(x)\hat{u}(x))+ \norm{x-y_2}^2 \leq 0$
    for all $x\in\Gamma_{y_1,y_2}$ and is bounded, since
    \begin{align*}
        \norm{\hat{u}(x)} &\leq \frac{ \norm{x-y_2}(\norm{f(x)} + \norm{x-y_2}) }{\norm{g(x)^T (x-y_2)}} \\
        &\quad \frac{\norm{ g(x)^{-1} g(x) (x-y_2) }(\norm{f(x)} + \norm{x-y_2}) }{\norm{g(x)^T (x-y_2)}} \\
        &\quad \leq \norm{g(x)^{-1}}\norm{x-y_2}.
    \end{align*}
    Given that an explicit expression for the CLF is available, the conditions~\ref{it:prob-comp-fourth},~\ref{it:prob-comp-fifth} in Lemma~\ref{lem:compat-neighboring-vertices-general} can be verified directly and one can choose the radius of the balls defining $\Nc_i$ to satisfy them.
    Furthermore, Propositions~\ref{prop:compat-verification-polytopic} and~\ref{prop:single-integrator-circular-obstacles} provide two settings where condition~\ref{it:prob-comp-sixth} holds.

    A similar argument can be made for the \emph{double integrator} in dimension $2k\in\mathbb{Z}_{>0}$. As mentioned in Remark~\ref{rem:controllability-requirements-clf-cbf-compatible-paths}, in that case only the points of the form $(x_f,\textbf{0}_k)\in\real^{2k}$ are stabilizable. Hence, the sets $\Nc_i$ in Lemma~\ref{lem:compat-neighboring-vertices-general} can be taken in the form $\Nc_i:=\setdef{(x,\textbf{0}_k)\in\real^{2k}}{\norm{x-x_f}<\nu_i}$ for some $\nu_i>0$.
    Furthermore, one can use the explicit expression of the CLF provided in Section~\ref{sec:compat-verif-high-rel-degree} and choose the parameters $\nu_i$ in order to verify the rest of the assumptions in Lemma~\ref{lem:compat-neighboring-vertices-general}.
    \demo
\end{remark}

In general, if the neighborhood $\Nc_i$ around $x_i$ in Lemma~\ref{lem:compat-neighboring-vertices-general} is sufficiently small and $\nabla V_y$ is continuous in $y$ (with the assumption that $V_{x_i} = V_i$),
%
%
the left-hand side of~\eqref{eq:prob-comp-CLF-gradient-condition} can be made sufficiently small so that the inequality holds. Note that Assumptions~\ref{it:prob-comp-second},~\ref{it:prob-comp-fifth}, and~\ref{it:prob-comp-sixth} are not restrictive and hold in several cases of interest, as outlined in Remark~\ref{rem:lemma-prob-com-assumptions-fully-actuated}.
%
%
%
Overall, the assumptions in Lemma~\ref{lem:compat-neighboring-vertices-general} ensure that there exist neighborhoods around every waypoint of a CLF-CBF compatible path such that the controller obtained as the solution of~\eqref{eq:opt-pb} can connect a point from each neighborhood to any point in the neighborhood of the next waypoint.
We next leverage this property to show the probabilistic completeness of \texttt{C-CLF-CBF-RRT}.

\begin{proposition}\longthmtitle{Probabilistic Completeness of \texttt{C-CLF-CBF-RRT}}\label{prop:prob-completeness-general}
    Suppose that there exists a CLF-CBF compatible path $\Ac=\{ x_i \}_{i=1}^{N_a}$, $N_a\in\mathbb{Z}_{>0}$, and suppose that all the assumptions 
    in Lemma~\ref{lem:compat-neighboring-vertices-general} regarding $\Ac$ hold.
    %
    %
    Further suppose 
    that 
    \begin{enumerate}
        \item\label{it:random-state-assumption} there exists a positive probabiliy $p_i$ that \texttt{RANDOM}$\underline{\hspace{0.2cm}}$\texttt{STATE} returns a point from $\Nc_i$;
        \item for each $y\in\Nc_i$, \texttt{FIND}$\underline{\hspace{0.2cm}}$\texttt{CLF} returns $V_y$ and $W_y$ (as defined in item~\ref{it:prob-comp-second} of Lemma~\ref{lem:compat-neighboring-vertices-general});
        \item\label{it:prob-comp-class-K-linear-bound} the extended class $\Kc_{\infty}$ functions $\{ \alpha_{i,l} \}_{i\in[N_a],l\in[M]}$ in~\eqref{eq:opt-pb} are upper bounded by linear extended class $\Kc_{\infty}$ functions, i.e., there exist $\hat{\alpha}_{i,l}>0$ for $i\in[N_a]$ and $l\in[M]$ such that $\alpha_{i,l}(s)\leq \hat{\alpha}_{i,l}s$ for all $s\geq0$;
        \item\label{it:prob-comp-steering-parameter} the steering parameter $\eta$ in 
         \texttt{NEW}$\underline{\hspace{0.2cm}}$\texttt{STATE}
        is such that $\eta>\max\limits_{i\in[N_a-1]} \max\limits_{\substack{y_2\in\Nc_{i+1}, y_1\in\Nc_i} } \norm{y_2-y_1}$.
    \end{enumerate}
    Then, there exists $\tau^*\in\mathbb{Z}_{>0}$ such that if  $\tau > \tau^*$, the probability of \texttt{C-CLF-CBF-RRT} (executed with parameters $\tau$, $\eta$, and any set of extended class $\Kc_{\infty}$ functions $\{ \alpha_l \}_{l\in[M]}$)
    %
    %
    returning a tree without a vertex in $\Xc_{\text{goal}}$ tends to zero as the number of iterations $k$ goes to infinity.
\end{proposition}
%
%
\begin{proof}
  The proof follows a similar reasoning to~\cite[Theorem 1]{MK-KS-ZL-KB-DH:19} that proves probabilistic completeness for \texttt{GEOM-RRT}.
  Let $i\in[N_a-1]$.
  First, we show that if $\Nc_i$ contains a vertex $x_{\text{near}}$ from the tree $\Tc$ in \texttt{C-CLF-CBF-RRT},
  then with probability $p_i>0$ in the next iteration a vertex will be added from $\Nc_{i+1}$.
  To see this, note that by assumption there exists a probability $p_i>0$ that the function \texttt{RANDOM}$\underline{\hspace{0.2cm}}$\texttt{STATE}
  returns a point  $x_{\text{rand}}$ from $\Nc_{i+1}$.
  Given (iv),
  the distance between $x_{\text{near}} \in \Nc_{i}$ and $x_{\text{rand}} \in \Nc_{i+1}$ is less than $\eta$, and therefore
  $x_{\text{new}}=x_{\text{rand}}$.
   Now, Lemma~\ref{lem:compat-neighboring-vertices-general} ensures that there exists a set of extended class $\Kc_{\infty}$ functions
  $\{ \bar{\alpha}_{i,l} \}_{l=1}^M$, a CLF $V_{x_{\text{rand}}}$ with respect to $x_{\text{rand}}$ and a positive definite function $W_{x_{\text{rand}}}^{\bar{\sigma}}$ with respect to $x_{\text{rand}}$ 
  such that \texttt{COMPATIBILITY}$(x_{\text{near}}, x_{\text{rand}},\tau,\{ h_l, \bar{\alpha}_{i,l} \}_{l=1}^M,V_{x_{\text{rand}}},W_{x_{\text{rand}}}^{\bar{\sigma}})$ returns \texttt{True}.
  Moreover, since the functions $\{ \alpha_{i,l} \}_{l=1}^M$ are upper bounded by linear extended class $\Kc_{\infty}$ functions with slopes $\{ \hat{\alpha}_{i,l} \}_{l=1}^M$, by performing the updates in the extended class $\Kc_{\infty}$ functions described in 
  Section~\ref{sec:compat-def}.3,
  it follows that there exists $\tau^*$ sufficiently large such that if $\tau>\tau^*$, the
  updated linear extended class $\Kc_{\infty}$ functions used in \texttt{COMPATIBILITY} have slopes larger than $\{ \hat{\alpha}_{i,l} \}_{l=1}^M$ respectively and the coefficient multiplying $W_{x_{\text{rand}}}$ is smaller than $\bar{\sigma}$, which makes the \texttt{COMPATIBILITY} function return \texttt{True}.
  This means that $x_{\text{rand}}$ is added to $\Tc$ with the corresponding edge from $x_{\text{near}}$ to $x_{\text{rand}}$, as stated.
  
  Next, in order for \texttt{C-CLF-CBF-RRT} to reach $\Xc_{\text{goal}}$ from $x_{\text{init}}$,
  the algorithm needs to successively select points from $\Nc_{i+1}$ as described
  previously for $i\in[N_a-1]$. For $k$ iterations of \texttt{C-CLF-CBF-RRT}, 
  this stochastic process can be described as $k$ Bernouilli
  trials~\cite[Definition 2.5]{RDY-DJG:04}
  %
  %
  with success probabilities $\{ p_i \}_{i=1}^{N_a-1}$. 
  The algorithm reaches $\Xc_{\text{goal}}$ from $x_{\text{init}}$
  after $N_a-1$ successful outcomes.
  Let $p:=\min\limits_{i\in[N_a-1]}p_i$.
  Using the same argument as in~\cite[Theorem 1]{MK-KS-ZL-KB-DH:19}, the probability that this stochastic process does not have $N_a-1$ successful outcomes after $k$ iterations is smaller than $\frac{(N_a-1)!}{(N_a-2)!}k^{N_a-1} e^{-p k}$.
  This means that the probability of \texttt{C-CLF-CBF-RRT} returning a tree without a vertex in $\Xc_{\text{goal}}$ tends to zero as the
  number of iterations $k$ goes to infinity.
\end{proof}


%
%
\begin{remark}\longthmtitle{Verification of Assumptions of Proposition~\ref{prop:prob-completeness-general}}\label{rem:prob-com-assumptions-fully-actuated}
{\rm
    As mentioned in Remark~\ref{rem:lemma-prob-com-assumptions-fully-actuated}, 
    for systems with the same number of inputs as state
      variables}, the set $\Nc_i$ in
    Lemma~\ref{lem:compat-neighboring-vertices-general} can be taken
    as a ball centered at the waypoint~$x_i$.  
    If \texttt{RANDOM}$\underline{\hspace{0.2cm}}$\texttt{STATE} samples $\Rc$ uniformly, it returns a point in such ball with probability equal to its relative volume in~$\Rc$.
    Furthermore, in this case \texttt{FIND}$\underline{\hspace{0.2cm}}$\texttt{CLF} can simply return $V_y(x) = \frac{1}{2}\norm{x-y}^2$ and $W_y(x) = \norm{x-y}^2$ for any $y\in\Nc_i$.
    For the \emph{double integrator}
    in dimension $2k\in\mathbb{Z}_{>0}$, as mentioned in Remark~\ref{rem:lemma-prob-com-assumptions-fully-actuated}, the sets $\Nc_i$ in Lemma~\ref{lem:compat-neighboring-vertices-general} can be taken in the form $\Nc_i:=\setdef{(x,\textbf{0}_k)\in\real^{2k}}{\norm{x-x_f}<\nu_i}$ for some $\nu_i>0$
    and if \texttt{RANDOM}$\underline{\hspace{0.2cm}}$\texttt{STATE} samples uniformly points of the form $(x_f,\textbf{0}_k)\in\real^{2k}$, then~\ref{it:random-state-assumption} in Proposition~\ref{prop:prob-completeness-general} holds.
    Furthermore, \texttt{FIND}$\underline{\hspace{0.2cm}}$\texttt{CLF} can return the explicit expression of the CLF provided in~\cite[Section V.A]{GY-CB-RT:19}.
    We note also that Assumption~\ref{it:prob-comp-class-K-linear-bound} is not restrictive, and Assumption~\ref{it:prob-comp-steering-parameter} holds by taking the parameter $\eta$ sufficiently large.
    \demo
\end{remark}

\begin{remark}\longthmtitle{Computational Complexity of \texttt{C-CLF-CBF-RRT}}
{\rm
  The computational complexity of \texttt{C-CLF-CBF-RRT} is the same as \texttt{GEOM-RRT} except for the 
  added complexity of the \texttt{COMPATIBILITY} function.
  In general, the optimization problems~\eqref{eq:r1-clf-bncbf-compat-check-general},~\eqref{eq:r2-clf-bncbf-compat-check-general}, and~\eqref{eq:opt-pb-check-collision-gamma}
  %
  %
  required by \texttt{COMPATIBILITY} can be non-convex, which makes them not computationally tractable.
  However, in the setting considered in Proposition~\ref{prop:compat-verification-polytopic}, 
  the worst-case complexity of \texttt{COMPATIBILITY}
  is that of solving $\tau$ QCQPs, 
  for which efficient heuristics
  exist~\cite{JP-SB:17-arxiv}.
  In the setting considered in Proposition~\ref{prop:single-integrator-circular-obstacles},~\eqref{eq:r1-clf-bncbf-compat-check-general},~\eqref{eq:r2-clf-bncbf-compat-check-general}, and~\eqref{eq:opt-pb-check-collision-gamma}
  can be solved in closed form, which means that \texttt{C-CLF-CBF-RRT} has the same computational complexity 
  as \texttt{GEOM-RRT}.
  \demo
  }
\end{remark}

  \begin{remark}\longthmtitle{\texttt{C-CLF-CBF-RRT} for 
      Differentially Flat
      Systems}\label{rem:differentially-flat-system} {\rm Here we 
      explain how \texttt{C-CLF-CBF-RRT} is applicable to
      differentially flat systems.  Differentially flat
      systems~\cite{MF-JL-PM-PR:92} are control systems for which the
      states and inputs can be written as algebraic functions of
      carefully selected \textit{flat outputs} and their derivatives.
      Many robotic systems of interest, such as the
      unicycle~\cite{DM-VK:11} or the quadrotor~\cite{LEB-AAM:24} are
      differentially flat.  This property facilitates the generation
      of smooth trajectories.  Differentially flat systems are
      equivalent to dynamic feedback linearizable
      systems~\cite{JL:07-ifac} (i.e., systems that can be feedback
      linearized after adding an appropriate number of dynamic
      inputs).  This means that differentially flat systems can be
      transformed into linear systems after an appropriate change of
      coordinates and control inputs (the same also applies to static
      feedback linearizable systems, for which no dynamic inputs need
      to be added).  Furthermore, by constructing an outer
      approximation of the obstacles using polytopes, and expressing
      it as a union of convex polytopes, the results in
      Proposition~\ref{prop:compat-verification-polytopic} apply, and
      the optimization problems (6) and (7) are easier to solve,
      cf. Section~\ref{sec:linear-systems-polytopic}.  \demo }
\end{remark}

%
%

\begin{remark}\longthmtitle{Controller
    Execution}\label{rem:execution-controller-finite-time}
  {\rm Given a CLF-CBF compatible path $\Ac$, executing the
    controller~\eqref{eq:opt-pb} has the agent converge from one
    waypoint to the next asymptotically. However, under the
    assumptions of Proposition~\ref{prop:prob-completeness-general},
    there exist neighborhoods around the waypoints of $\Ac$ such that
    any two points of two consecutive neighborhoods can be connected
    with a CLF-CBF controller (possibly, with adjusted CLF, and
    extended class $\Kc_{\infty}$ functions,
    cf. Lemma~\ref{lem:compat-neighboring-vertices-general}). Therefore,
    by executing the controller~\eqref{eq:opt-pb} for a sufficiently
    large but finite time, the agent can visit these different
    neighborhoods and trace a path whose waypoints are close to those
    of~$\Ac$.  \demo }
\end{remark}

\begin{remark}\longthmtitle{\texttt{C-CLF-CBF-RRT} for Higher-Relative Degree Systems}\label{rem:c-clf-cbf-rrt-high-relative-degree-systems}
{\rm
    \texttt{C-CLF-CBF-RRT} can be adapted to the setting where $h$ is a HOCBF, cf. Section~\ref{sec:compat-verif-high-rel-degree}, with the following modifications:
    \begin{enumerate}
      \item $x_{\text{init}}$ and $\Xc_{\text{goal}}$ lie in $\Cc\cap\Cc_2\cap\hdots\cap\Cc_{\bar{m}}$;
      \item \texttt{RANDOM}$\underline{\hspace{0.2cm}}$\texttt{STATE} returns states from $\Cc\cap\Cc_2\cap\hdots\cap\Cc_{\bar{m}}$ (or a subset of it consisting of stabilizable points);
      \item \texttt{COMPATIBILITY} employs the conditions described in Proposition~\ref{prop:clf-hocbf-compat-characterization} instead of those in Proposition~\ref{prop:clf-bncbf-compat-general}
      to check the compatibility of CLFs and HOCBFs.     \demo
    \end{enumerate}
    }
\end{remark}



\section{Simulation and Experimental Validation}\label{sec:experimental-validation}

Here we illustrate the performance of \texttt{C-CLF-CBF-RRT} in simulation and hardware experiments.
Throughout the section, we deal with a differential-drive robot following the unicycle dynamics:
\begin{subequations}\label{eq:unicycle}
\begin{align}
  \dot{x} &= v \cos(\theta), \\
  \dot{y} &= v \sin(\theta), \\
  \dot{\theta} &= \omega,
\end{align}
\end{subequations}
where $s=[x,y]\in\real^2$ is the position of the robot, $\theta$ its heading,
and $v$ and $\omega$ are its linear and angular velocity control inputs, respectively.
Following~\cite[Section IV]{PG-IB-ME:19}, we set
\begin{align*}
    R(\theta)=\begin{bmatrix}
        \cos{\theta} & -\sin{\theta} \\
        \sin{\theta} & \cos{\theta}
    \end{bmatrix}, \quad p=\begin{bmatrix}
        x \\
        y
    \end{bmatrix} + l_0 R(\theta) e_1,
\end{align*}
where $e_1=[1, 0]^T$ and $l_0>0$ is a design parameter. 
This defines $p$ as a point orthogonal to the wheel axis of the robot.
Moreover, let
\begin{align*}
    L = \begin{bmatrix}
        1 & 0 \\
        0 & 1/l_0
        \end{bmatrix}.
\end{align*}
Even though the dynamics~\eqref{eq:unicycle} are nonlinear,
it follows that $\dot{p}=R(\theta)L^{-1}u$, where $u=[v,w]^T$.
By defining the new control input $\tilde{u}=R(\theta) L^{-1} u$, the state $p$ follows
single integrator dynamics. The original angular and linear
velocity inputs can be easily obtained from $\tilde{u}$ as $u = L R(\theta)^{-1} \tilde{u}$.
Since $p$ can be made arbitrarily
close to $[x,y]$ by taking $l_0$ sufficiently small, in what follows we consider $p$ as our state variable.
Throughout the experiments, 
we use $\alpha_l(s) = 5s$ for all $l\in[M]$ and $\eta=2m$. 
We also use $\tau = 5$ and constants $\sigma = 0.5$, $\bar{\sigma}=2$ as defined in 
Section~\ref{sec:c-clf-cbf-rrt}. The results we present in this section have been obtained without 
the need to resort to increase the value of $\sigma$ or $\bar{\sigma}$ at every iteration,
or perform other similar heuristics.
Once the robot is within $0.5m$ of a given waypoint, we switch the controller so that it 
steers the robot towards the next waypoint.

\subsection{Computer Simulations}\label{sec:computer-simulations}

We have tested \texttt{C-CLF-CBF-RRT} in different simulation environments in a high-fidelity Unity simulator on an Ubuntu PC with
Intel Core i9-13900K 3 GHz 24-Core processor.
We utilize the function \texttt{minimize} from the library \texttt{SCIPY}~\cite{PV-RG-TEO:20}
to solve the optimization problems in the \texttt{COMPATIBILITY} function.
The robots used in the simulation are Clearpath Husky\footnote{Spec. sheets for the Husky and Jackal robots can be found at https://clearpathrobotics.com} robots, which have the same LIDAR and sensor capabilities as the
real ones, and these are used to run a SLAM system that
allows each robot to localize itself in the environment and
obtain its current state, which is needed to implement the controller from~\eqref{eq:opt-pb}.
The first simulation environment consists of a series of red obstacles whose projection on the navigation
plane is either a circle or a polytope.
%
%
%
%
%
The second simulation consists of an environment with different rooms. The different walls are modeled as obstacles 
using nonsmooth CBFs, given that their projection on the navigation plane are quadrilaterals. To ensure that the whole physical body of the robot remains safe, we add a slack term to the CBF that takes into account the robot dimensions. For example, for a circular obstacle with center at $x_c\in\real^2$ and radius $r>0$, and a circular robot with radius $r_0>0$, the CBF can be taken as $h(x) = \norm{x-x_c}^2-(r+r_0)^2$.
Both simulation environments have dimensions $20m\times50m$, and in each of them the projection of the obstacles in the navigation plane is either a circle or a polytope, so the \texttt{COMPATIBILITY} function runs efficiently (cf. Section~\ref{sec:compat-verification}).
Figure~\ref{fig:nonconvex-env} shows the tree generated by \texttt{C-CLF-CBF-RRT} in both simulation experiments, as well as the corresponding trajectory
executed by the robot using the controller obtained as the solution of~\eqref{eq:opt-pb},
%
%
which successfully reaches the end goal while remaining collision-free.

\begin{figure}[htb]
  \centering
  \subfigure[Environment with obstacles]{\includegraphics[width=0.85\linewidth]{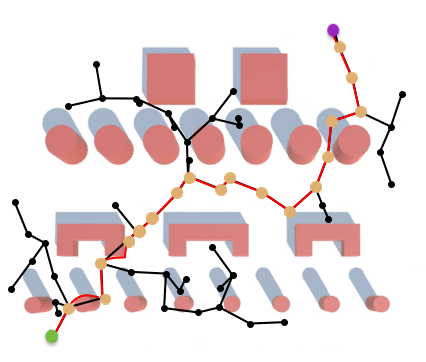}}
  \\
  \subfigure[Environment with rooms]{
  \includegraphics[width=0.85\linewidth]{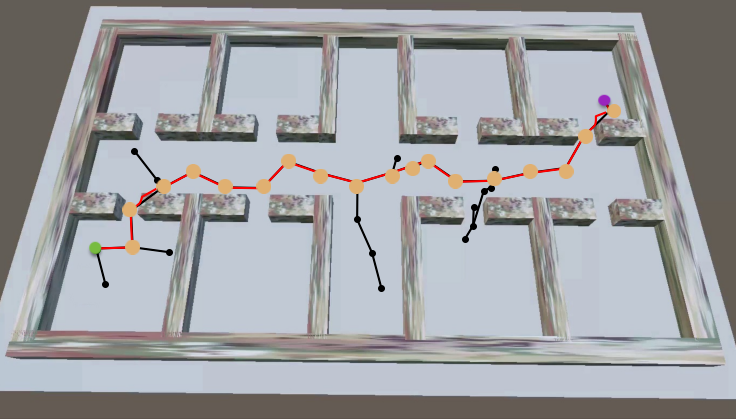}}
  \caption{(a) First and (b) second simulation environment experiments.
  Tree generated by \texttt{C-CLF-CBF-RRT} (black), waypoints of the returned path (dark yellow) and trajectory followed by the robot using 
  the controller from~\eqref{eq:opt-pb} (red). The starting point is the green dot and the end goal is 
  the purple dot.
  In each environment, the robot successfully visits the waypoints while avoiding collisions with obstacles.}
  \label{fig:nonconvex-env}
\end{figure}
%
%

\subsection{Hardware Experiments}

%
%
We have also tested \texttt{C-CLF-CBF-RRT} in a physical environment using a Clearpath Jackal robot. The robot is equipped with GPS, IMU and LIDAR sensors, which are used to run a SLAM system to localize its position in the environment and execute the controller from~\eqref{eq:opt-pb}. The environment, with dimensions $4m\times 9m$, consists of different obstacles whose projection on the navigation plane is either a circle or a polytope. We ensure the whole physical body of the robot remains safe using a slack term in the CBF formulation, as described in Section~\ref{sec:computer-simulations}. Figure~\ref{fig:hardw2-env}(a) shows the tree generated by \texttt{C-CLF-CBF-RRT} as well as the trajectory executed by the robot, successfully reaching its goal. We use $\alpha_l(s) = 5s$ for all $l\in[M]$ and choose $\eta=2m$. Once the robot is within $0.5m$ of a given waypoint, we switch the controller so that it steers the robot towards the next waypoint.

\begin{figure}[htb]
  \centering
  \subfigure[\texttt{C-CLF-CBF-RRT}]{
  \includegraphics[width=0.85\linewidth]{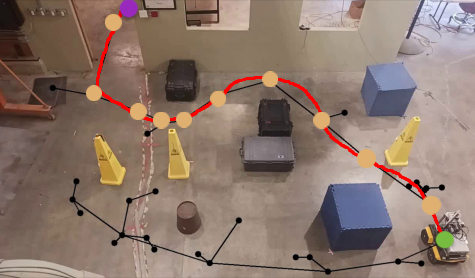}}
 \\
 \subfigure[\texttt{GEOM-RRT}]{
  \includegraphics[width=0.85\linewidth]{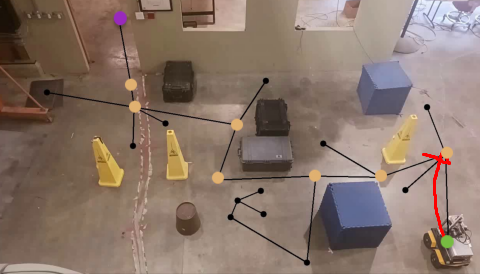}}
  \caption{Execution of (a) \texttt{C-CLF-CBF-RRT} and (b) \texttt{GEOM-RRT} in the hardware experiment. In both plots, tree generated by the corresponding algorithm (black), waypoints of the returned path (dark yellow), and trajectory followed by the robot (red) using 
  the controller from~\eqref{eq:opt-pb} (red). The starting point is the green dot and the end goal is 
  the purple dot. 
  The trajectory executed by the robot under \texttt{C-CLF-CBF-RRT} reaches its goal safely, whereas 
  it fails under \texttt{GEOM-RRT} because it quickly encounters a point where the optimization problem~\eqref{eq:opt-pb} is infeasible.}
  \label{fig:hardw2-env}
\end{figure}

\subsection{Comparison with \texttt{GEOM-RRT}}
Here we compare the performance of \texttt{C-CLF-CBF-RRT} with \texttt{GEOM-RRT} in both the simulation and hardware environments.
Figure~\ref{fig:hardw2-env}(b)
%
%
shows the tree generated by \texttt{GEOM-RRT} as well as the trajectory 
executed by the robot in the hardware environment using the controller obtained from~\eqref{eq:opt-pb}.
%
%
One can observe that the trajectory generated by the robot is unable to reach
the end goal and stops rather early, at a point where the optimization problem~\eqref{eq:opt-pb} becomes
infeasible. This occurs because \texttt{GEOM-RRT} does not 
take into account the dynamic feasibility of the path it generates.

We should point out that the steering parameter $\eta$ critically affects the performance of \texttt{GEOM-RRT}.
To show this, we run various executions of \texttt{GEOM-RRT} in the simulation environment with obstacles depicted in Figure~\ref{fig:nonconvex-env}(a).
Table~\ref{tab:geom-rrt-comparison} shows that smaller values of $\eta$ yield a higher percentage of feasible 
paths but with a higher average execution time. 
For comparison, the average execution time of  \texttt{C-CLF-CBF-RRT}, whose paths are always dynamically feasible, for the same simulation environment and with $\eta=4m$, is 8.72 seconds.
%
%
To match the dynamic feasibility of the produced paths, \texttt{GEOM-RRT} has to be run with $\eta=1m$, at a significantly higher computational cost.
\begin{table}[htb]
  \centering
  \begin{tabular}{| c | c | c |}
    \hline
    $\eta$ (meters) & \begin{tabular}{@{}c@{}} Percentage of \\ feasible paths \end{tabular} & \begin{tabular}{@{}c@{}} Average execution \\ time (seconds) \end{tabular} \\
    \hline
    1 & 100$\%$ & 154.36 \\
    \hline 
    2 & 90$\%$ & 140.62 \\
    \hline
    4 & 50$\%$ & 130.62 \\
    \hline 
    8 & 30$\%$ & 4.83 \\
    \hline 
    16 & 5$\%$ & 1.84 \\
    \hline
  \end{tabular}
  \caption{Comparison of the percentage of feasible paths (i.e., paths for which 
  the controller in~\eqref{eq:clf-bncbf-controller} steers the robot from the initial point to the 
  end goal by following the waypoints generated by the path) and the average execution time of \texttt{GEOM-RRT} (over 20 executions).
  The paths are generated for the simulation environment with obstacles in Figure~\ref{fig:nonconvex-env}(a).
  }\label{tab:geom-rrt-comparison}
\end{table}
%
%

  \begin{remark}\longthmtitle{Convergence to
      waypoints}\label{rem:tuning-waypoint-convergence-rate}
    \rm{Since the robot asymptotically converges to each waypoint, we
      observe in the experiments that it tends to slow down when
      reaching a waypoint and speed up when switching to the next
      one. Here we describe ways in which this behavior can be
      alleviated:
  \begin{enumerate}
  \item \textbf{Modifying the objective function}: The minimum-norm
    controller in~\eqref{eq:opt-pb} naturally seeks the smallest
    control action, which can lead to the observed \textit{slowing
      down} effect near waypoints. Alternatively, given a nominal
    controller $u_{\text{nom}}:\real^n\to\real^m$ with a desired
    behavior (towards a waypoint or the end goal), one can modify the
    objective function in~\eqref{eq:opt-pb} by
    $\frac{1}{2}\norm{u-u_{\text{nom}}(x)}^2$ and implement the
    resulting controller.
    %
    %
    
  \item \textbf{Finite-time CLFs}: Fixed-time Control Lyapunov
    Function~\cite{KG-EA-DP:22} can be used to design controllers that
    guarantee convergence to a desired waypoint within a specified
    time horizon. By extending the notion of compatibility to consider
    BNCBFs and finite-time CLFs,
    the optimization
    problems~\eqref{eq:r1-clf-bncbf-compat-check-general}
    and~\eqref{eq:r2-clf-bncbf-compat-check-general} can be
    reformulated using finite-time CLFs.  Then, these optimization
    problems can be utilized to define a version of
    \texttt{C-CLF-CBF-RRT} that accounts for finite-time CLFs. We
    leave the study of the properties of such an algorithm for future
    work.
    %
    %
    
  \item \textbf{CLF convergence rate}: given a closed-loop system
    satisfying the CLF condition~\eqref{eq:clf-ineq},
    %
    %
    the function $W$ dictates the rate of decrease of trajectories to
    the origin. For example, by taking $W(x) = \gamma V(x)$, with
    $\gamma>0$, trajectories of the closed-loop system converge to the
    origin at a rate $\gamma$.  Therefore, by increasing $\gamma$, the
    rate of convergence can be increased.  However, it should be
    pointed out that an increased rate of convergence might compromise
    the compatibility of $V$ with a CBF.     \demo
      %
      %
    \end{enumerate}}
\end{remark}

\subsection{Comparison with \texttt{CBF-RRT} and \texttt{LQR-CBF-RRT*}}
%
%
Here we compare \texttt{C-CLF-CBF-RRT} with other related algorithms
in the literature leveraging CBFs.  First, we compare it with
\texttt{CBF-RRT}, a sampling-based motion planning algorithm proposed
in~\cite{GY-BV-ZS-CB-RT:19} that also employs control barrier
functions. Initially, \texttt{CBF-RRT} starts with a tree consisting
of a single node in $x_{\text{init}}$.  Then, each iteration of
\texttt{CBF-RRT} operates as follows. First, it randomly samples a
vertex $x_0$ from the current tree. Next, it generates a reference
input, e.g., one steering the robot from $x_0$ to the goal set
$\Xc_{\text{goal}}$ (cf.~\cite[Section 5]{GY-BV-ZS-CB-RT:19} for more
details).  Finally, for a fixed period of simulated time $T_0$,
at every state it executes the controller closest to the reference
input that satisfies the CBF conditions associated to all
  obstacles.  This quadratic optimization program is solved using the
  convex optimization library \texttt{CVXOPT}~\cite{MSA-JD-LV:13}.
The state $x_{\text{new}}$ reached by the robot after this period of
time $T_0$ gets added to the tree.
%
%
%
%
To generate the trajectory, we numerically integrate the
  closed-loop system using the \texttt{odeint} method from the Python
  library \texttt{SCIPY}~\cite{PV-RG-TEO:20} and use a time
  discretization step of $0.005$ seconds.  We have ran multiple times
\texttt{C-CLF-CBF-RRT} and \texttt{CBF-RRT} in the simulation
environment with obstacles of Figure~\ref{fig:nonconvex-env}(a). Note
that \texttt{CBF-RRT} is more computationally costly, as it requires
running a trajectory for every new node added to the
tree. Furthermore, this trajectory is generated by a controller that
is obtained as the solution of an optimization problem at every
point. In contrast, \texttt{C-CLF-CBF-RRT} only requires solving a
single optimization problem (and, in the cases discussed in
Section~\ref{sec:integrator-circular}, not even that, since an
algebraic check is enough) for every new node added to the tree.  For
example, if $T_0$ is small (e.g., $T_0=5$), the average execution time
of \texttt{CBF-RRT} exceeds one minute.  For $T_0 = 15$, the average
execution time of \texttt{CBF-RRT} (over 10 different runs) is 384.58
seconds.  The average execution time is similar for $T_0 = 10$,
$T_0 = 20$.
These numbers seem to indicate that smaller values of $T_0$ find a
feasible path more rapidly, but such paths contain a larger number of
waypoints. In contrast, larger values of $T_0$ lead to paths with a
smaller number of waypoints but require more time to be found.  In
comparison, the average execution time of \texttt{C-CLF-CBF-RRT} with
the same initial point and end goal (and with $\alpha_l(s)=5s$ for all
$l\in[M]$ and $\eta=4m$) is 8.72 seconds, almost two orders of
magnitude faster. We should also point out that there exists a
  trade-off between the computational complexity of CBF-RRT and the
  underlying safety guarantees.  Indeed, since the CBF-QP controller
  cannot be solved continuously, CBF-RRT~\cite{GY-BV-ZS-CB-RT:19}
  solves the CBF-QP optimization problem periodically along the
  generated trajectory with sampling time~$T_0$.  As a consequence,
  in-between the times when the CBF-QP is solved, safety violations
  may occur.  One way to remedy this is to solve the CBF-QP at a
  higher frequency. However, this increases the computational
  complexity of CBF-RRT, since the overall number of optimization
  problems to be solved is
  higher.  
%
%
%
%
%
%

  Finally, we compare \texttt{C-CLF-CBF-RRT} with
  \texttt{LQR-CBF-RRT*}. This is a sampling-based algorithm proposed
  in~\cite{GY-MC-AA-AP-RT-CB:23} which generates reference
  trajectories using LQR-based controllers of linearized dynamics
  around a new added node to the RRT, and checks the CBF condition at
  a finite set of points along this reference trajectory. The
  resolution with which such CBF condition is checked affects the
  safety of the overall trajectory (theoretically, it is safe only if
  every point satisfies the CBF condition).
  Table~\ref{tab:lqr-comparison} compares different resolutions with
  which the CBF condition checks are made, along with the
  corresponding average execution times (over 20 runs) and safety
  violations (which, to have a fair comparison with
  \texttt{C-CLF-CBF-RRT}, has been implemented without the adaptive
  sampling procedure described in~\cite[Section
  V.C]{GY-MC-AA-AP-RT-CB:23}).  Smaller resolutions naturally lead to
  larger execution times and a smaller percentage of safety
  violations.
  We note that a resolution of $0.01m$ leads to no safety violations
  and only has a slightly higher execution time compared to
  \texttt{C-CLF-CBF-RRT}. However, this lack of safety violations is
  not theoretically guaranteed in general and it is not known a
  priori what resolution results in no safety violations.
%
%

\begin{table}[htb]
  \centering
  \small
  \begin{tabular}{ | c | c | c |}
    \hline
    Resolution (m) & Average execution time (s) & Safety violations \\
    \hline
    0.5 & 0.26 & 60 $\%$ \\
    \hline
    0.1 & 1.09 & 40 $\%$ \\
    \hline
    0.05 & 1.7 & 5 $\%$ \\
    \hline
    0.01 & 11.31 & 0 $\%$ \\
    \hline
  \end{tabular}
  \caption{Comparison of the resolution with which the CBF checks
      are made in \texttt{LQR-CBR-RRT*} and the corresponding average
      execution time (over 20 executions).  The paths are generated
      for the simulation environment with obstacles in Figure
      2(a).}\label{tab:lqr-comparison}
\end{table}



\section{Conclusions}
We have introduced \texttt{C-CLF-CBF-RRT}, a sampling-based motion
planning algorithm that generates dynamically feasible collision-free
paths from an initial point to an end goal.
The algorithm creates a sequence of waypoints and results in a
well-defined CLF-CBF-based controller that generates trajectories
guaranteed to be safe and to sequentially visit the waypoints.
These guarantees are based on a result of independent interest that
shows that the problem of verifying whether a CLF and a BNCBF are
compatible in a set of interest can be solved by finding the optimal
value of an optimization problem. For systems with linear dynamics,
quadratic CLFs, and CBFs of polytopic or ellipsoidal
obstacles, this optimization problem is a QCQP, and for CBFs of
circular obstacles, it can be solved in closed form. In
  these scenarios, this ensures that \texttt{C-CLF-CBF-RRT} can be
  executed efficiently. Finally, we have shown that
\texttt{C-CLF-CBR-RRT} is probabilistically complete and can be
generalized to systems where safety constraints have a high relative
degree.  Simulations and hardware experiments illustrate the
performance and computational benefits of \texttt{C-CLF-CBR-RRT}.
Future work will explore the extension of the results to other
sampling-based algorithms (e.g., RRT*, bidirectional RRT),
construct asymptotically optimal versions of
  \texttt{C-CLF-CBR-RRT}, and integrate available computational tools
  to find CLFs and verify the compatibility of CLF-CBF pairs.
  We also plan to explore techniques to alleviate the observed
  \textit{slowing down} effect near waypoints,
consider systems under uncertainty, both in the robot dynamics and the
obstacles in the environment, and extend hardware implementations
  of \texttt{C-CLF-CBF-RRT} to more complex systems exploiting the
  notion of differential flatness.

\section*{Acknowledgments}
This work was supported by the Tactical Behaviors for Autonomous
Maneuver (TBAM) ARL-W911NF-22-2-0231 and W911NF-25-2-0042. Part of
this work was conducted during an internship by the first author at
the U.S. Army Combat Capabilities Development Command Army Research
Laboratory in Adelphi, MD during the summer of 2024.

%
%
\bibliography{../bib/alias,../bib/JC,../bib/Main-add,../bib/Main} 
\bibliographystyle{IEEEtran}

\appendix

  The following result shows that the problem of checking whether
  the optimization problem~\eqref{eq:opt-pb-xnear-compat-check} is
  feasible can be simplified by only checking the pairwise feasibility
  of the CLF constraint and the CBF constraints associated with the
  obstacles that intersect with $\Theta$ (as defined in
  Section~\ref{sec:compat-def}).
  \begin{lemma}\longthmtitle{Checking pairwise compatibility of a
      reduced set of CBFs}\label{lem:reduction-cbf-set}
    Let $x_{\text{near}}\in\real^n$, $x_{\text{new}}\in\real^n$, and
    $V:\real^n\to\real$ be a CLF with respect to $x_{\text{new}}$.
    Define $\Theta=\setdef{x\in\real^n}{V(x)\leq V(x_{\text{near}})}$.
    Let $\Lc:=\setdef{l\in[M]}{\Theta\cap\Cl(\Oc_l)=\emptyset}$.
    Suppose that there exists a set of extended class $\Kc_{\infty}$
    functions $\{ \alpha_l \}_{l\in\Lc}$ such that for each $l\in\Lc$,
    the problem
    \begin{align}\label{eq:opt-pb-xnear-mathcalL}
      &\min_{u\in\real^m} \frac{1}{2}\norm{u}^2
      \\
      \notag
      &\quad \text{s.t.} \ L_fh_{j,l}(x) + L_gh_{j,l}(x)u \geq
        -\alpha_l(h_{j,l}(x)), \ j\in\Ic_l(x),
      \\
      \notag
      &\quad \quad \ L_fV(x) + L_gV(x)u + W(x) \leq 0,
    \end{align}
    is feasible for all $x\in\Theta\cap\Fc$ and there exists a set of
    disjoint open sets $\{ \Yc_l \}_{ l\in\Lc }$ (with $\Yc_l$ being a
    neighborhood of $\partial\Oc_l$ satisfying
    $\Yc_l\cap\Oc_{l^\prime} = \emptyset$ for all $l^\prime \neq l$)
    and a bounded controller $\hat{u}$ satisfying the constraints
    in~\eqref{eq:opt-pb-xnear-mathcalL} for each
    $x\in\Yc_l\cap\Theta\cap\Fc$ and $l\in\Lc$.  Then, there exists a
    set of extended class $\Kc_{\infty}$ functions
    $\{ \bar{\alpha}_l \}_{l\in[M]}$ such that
    \begin{align}\label{eq:opt-pb-xnear-Theta-Fc}
      &\min_{u\in\real^m} \frac{1}{2}\norm{u}^2
      \\
      \notag
      &\quad \text{s.t.} \ L_fh_{j,l}(x) + L_gh_{j,l}(x)u \geq
        -\bar{\alpha}_l(h_{j,l}(x)),
      \\
      \notag
      &\qquad \qquad \qquad \qquad \qquad \qquad \forall j\in\Ic_l(x),
        l\in\Lc,
      \\
      \notag
      &\quad \quad \ L_fV(x) + L_gV(x)u + W(x) \leq 0.
    \end{align}
    is feasible for all $x\in\Theta\cap\Fc$.
  \end{lemma}
  \begin{proof}
    Let $l\in\Lc$.  Note that since
    $\Yc_l\cap\Oc_{l^\prime}=\emptyset$ for all
    $l^\prime \in [M]\backslash\{ l \}$, there exists $d_l > 0$ such
    that $h_{l^\prime}(x) \geq d_l$ for all
    $l^\prime \in [M]\backslash\{ l \}$ and
    $x\in\mathcal{Y}_{l}\cap\Theta\cap\Fc$.  Now, take
    $\hat{\alpha}_l > 0$ such that
  \begin{align*}
    \hat{\alpha}_l > \frac{\sup\limits_{x\in\Yc_l\cap\Theta\cap\Fc}
    |L_f h_{j,l^{\prime}}(x) + L_gh_{j,l^{\prime}}(x)
    \hat{u}(x)|}{d_l} 
  \end{align*}
  for all $l^\prime \in [M]\backslash\{ l \}$ and
  $j\in\Ic_{l^{\prime}}(x)$.  Note that such $\hat{\alpha}_l$ exists
  because $\hat{u}$ is bounded and $\Theta$ is compact.  Further let
  $\hat{\alpha} > \hat{\alpha}_l$ for all $l\in\Lc$, and take
  $\bar{\alpha}_l$ so that
  $\bar{\alpha}_l(s) > \max\{ \alpha_l(s), \hat{\alpha}s \}$ for all
  $s\geq 0$.  Now, $\hat{u}(x)$ is feasible
  for~\eqref{eq:opt-pb-xnear-Theta-Fc} for any
  $x\in\big( \bigcup_{l\in\Lc} \Yc_l \big) \cap \Theta\cap\Fc$.  On
  the other hand, there exists $d_{-1} > 0$ such that
  $h_l(x) > d_{-1}$ for all $\l\in[M]$ and
  $x\in \Theta\cap\Fc \backslash \Big( \bigcup_{l\in\Lc} \Yc_l \Big)$.
  Now, take $\hat{\alpha}_{-1}>0$ such that
  \begin{align*}
    \hat{\alpha}_{-1} > \frac{\sup\limits_{x\in \Theta\cap\Fc
    \backslash ( \bigcup_{l\in\Lc} \Yc_l ) } |L_f h_{j,l}(x) +
    L_gh_{j,l}(x) \hat{u}(x)|}{d_{-1}}, 
  \end{align*}
  for all $l\in[M]$ and $j\in\Ic_l(x)$.  Again, such
  $\hat{\alpha}_{-1}$ exists because $\hat{u}$ is bounded and $\Theta$
  is compact.  Further let
  $\hat{\alpha}_*>\max\{ \hat{\alpha}, \hat{\alpha}_{-1} \}$ and take
  $\bar{\alpha}_l$ so that
  $\bar{\alpha}_l(s) > \max\{ \alpha_l(s), \hat{\alpha}_* s \}$ for
  all $s\geq 0$.  Hence, $\hat{u}(x)$ is feasible
  for~\eqref{eq:opt-pb-xnear-Theta-Fc} for any $x\in\Theta\cap\Fc$.
\end{proof}

\begin{table}[htb]
  \centering
  \small
  \begin{tabular}{ | c | c |}
    \hline
    \textbf{Symbol} & \textbf{Meaning} \\
    \hline
    \multicolumn{2}{c|}{ \textbf{Dynamics} } \\
    \hline
    $n$ & state dimension \\
    \hline
    $m$ & input dimension \\
    \hline
    $f$, $g$ & vector fields defining dynamics \\
    \hline
    \multicolumn{2}{c|}{ \textbf{CBF definitions} } \\
    \hline
    $M$ & number of obstacles \\
    \hline
    $\Oc_l$ & $l$-th obstacle \\
    \hline
    $h_l$ & BNCBF associated with $l$-th obstacle \\
    \hline
    $\Ic_l(x)$ & set of active indices for $h_l$ at $x\in\real^n$ \\
    \hline
    $\Fc$ & safe space \\
    \hline
    $N_l$ & number of functions defining $h_l$ \\
    \hline
    $h_{i,l}$ & $i$-th function defining $h_l$ \\
    \hline
    $\bar{m}$ & relative degree of a HOCBF \\
    \hline
    $\Zc_{l,\Jc}$ & \begin{tabular}{@{}c@{}} set of points where active constraints for \\ $\Oc_l$ have indices in $\Jc$ \end{tabular} \\
    \hline
    $\{ \beta_i \}_{i\in\Jc}$ & optimization variables in~\ref{eq:r1-clf-bncbf-compat-check-general},~\ref{eq:r2-clf-bncbf-compat-check-general} \\
    \hline
    \multicolumn{2}{c|}{ \textbf{CLF-CBF compatible path} } \\
    \hline
    $\Ac = \{ x_i \}_{i=1}^{N_a}$ & CLF-CBF compatible path \\
    \hline
    $N_a$ & number of waypoints in $\Ac$ \\
    \hline
    $x_{\text{init}}$ & starting point of $\Ac$ \\
    \hline
    $\Xc_{\text{goal}}$ & goal set \\
    \hline
    $V_i$ & CLF with respect to waypoints $x_{i+1}$ \\
    \hline
    $\Gamma_i$ & $\setdef{x\in\real^n}{V_i(x)\leq V_i(x_{i+1})}$ \\
    \hline
    $W_i$ & positive definite associated with $V_i$ \\
    \hline
    $\alpha_{i,l}$ & \begin{tabular}{@{}c@{}} extended class $\Kc_{\infty}$ associated to \\ waypoint $x_{i+1}$ and $\Oc_l$ \end{tabular} \\
    \hline
    \multicolumn{2}{c|}{ \textbf{\texttt{C-CLF-CBF-RRT}} } \\
    \hline
    $\Tc$ & tree constructed in \texttt{C-CLF-CBF-RRT} \\
    \hline
    $x_{\text{rand}}$ & randomly sampled node \\
    \hline
    $x_{\text{new}}$ & new node to be added to $\Tc$ \\
    \hline
    $x_{\text{near}}$ & nearest node in $\Tc$ to $x_{\text{new}}$ \\
    \hline
    $k$ & number of iterations of \texttt{C-CLF-CBF-RRT} \\
    \hline
    $\eta$ & steering parameter in \texttt{NEW}$\underline{\hspace{0.2cm}}$\texttt{STATE} \\
    \hline
    $\tau$ & \begin{tabular}{@{}c@{}} number of updates of $\{ \alpha_l \}_{l\in[M]}$ \\ and $W$ in \texttt{COMPATIBILITY} \end{tabular} \\
    \hline 
    $V, W$ & \begin{tabular}{@{}c@{}} CLF and associated positive definite function \\ returned by \texttt{FIND}$\underline{\hspace{0.2cm}}$\texttt{CLF} \end{tabular} \\
    \hline
    \multicolumn{2}{c|}{ \textbf{\texttt{COMPATIBILITY}}  } \\
    \hline
    $\Theta$ & $\setdef{x\in\real^n}{V(x)\leq V(x_{\text{near}})}$ \\
    \hline
    $\Lc$ & $\setdef{l\in[M]}{ \text{Cl}(\Oc_l)\cap\Theta\neq\emptyset }$ \\
    \hline
    $\zeta_{1,l}$ & solution of~\eqref{eq:r1-clf-bncbf-compat-check-general} for $l\in\Lc$ \\
    \hline
    $\zeta_{2,l}$ & solution of~\eqref{eq:r2-clf-bncbf-compat-check-general} for $l\in\Lc$ \\
    \hline
    \multicolumn{2}{c|}{ \textbf{Notation in Proposition~\ref{prop:prob-completeness-general}} } \\
    \hline
    $\delta_{\text{clear}}$ & positive number such that $\Bc(x_i,\delta_{\text{clear}})\subset\Fc$ \\
    \hline
    $\Nc_i$ & neighborhood of waypoint $x_i$ in Lemma~\ref{lem:compat-neighboring-vertices-general} \\
    \hline
    $\hat{\Gamma}_i$ & superset of $\Gamma_i$ in Lemma~\ref{lem:compat-neighboring-vertices-general} \\
    \hline
    $\Gamma_{y_1,y_2}$ & $\setdef{x\in\real^n}{V_{y_2}(x) \leq V_{y_2}(y_1)}$ \\
    \hline
    $\Zc$ & $\Zc=\setdef{z\in\Fc}{\exists l\in[M] \ \text{s.t.} \ d(z,\Oc_l)\leq \tfrac{\delta_{\text{clear}}}{2} }$ \\
    \hline
    \multicolumn{2}{c|}{ \textbf{Simulations} } \\
    \hline
    $R(\theta)$ & rotation matrix of angle $\theta$ \\
    \hline
    $v, \omega$ & \begin{tabular}{@{}c@{}} linear and angular velocities \\ in unicycle model \end{tabular} \\
    \hline
  \end{tabular}
  \caption{Summary of symbols used in the paper.}
  \label{tab:1}
\end{table}
%
%


\vspace*{-3ex}

\begin{IEEEbiography}[{\includegraphics[width=1in,clip,keepaspectratio]{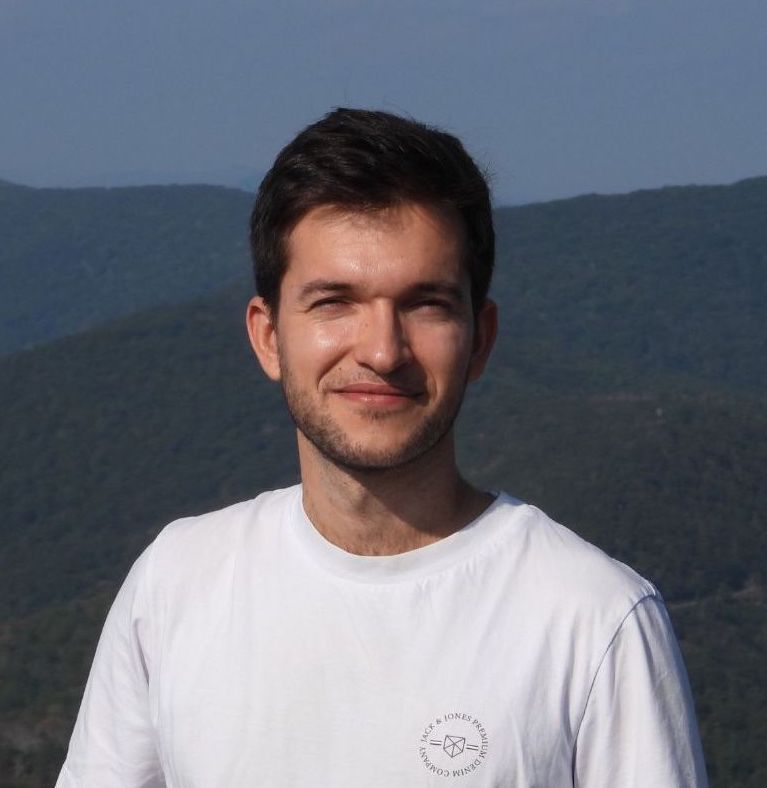}}]{Pol
    Mestres} received the Bachelor's degree in mathematics and the
  Bachelor's degree in engineering physics from the Universitat
  Polit\`{e}cnica de Catalunya, Barcelona, Spain, in 2020, and the
  Master's degree in mechanical engineering in 2021 from the
  University of California, San Diego, La Jolla, CA, USA, where he is
  currently a Ph.D. candidate. His research interests include
  safety-critical control, optimization-based controllers, distributed
  optimization and motion planning.
\end{IEEEbiography}

\begin{IEEEbiography}[{\includegraphics[width=1in,clip,keepaspectratio]
    {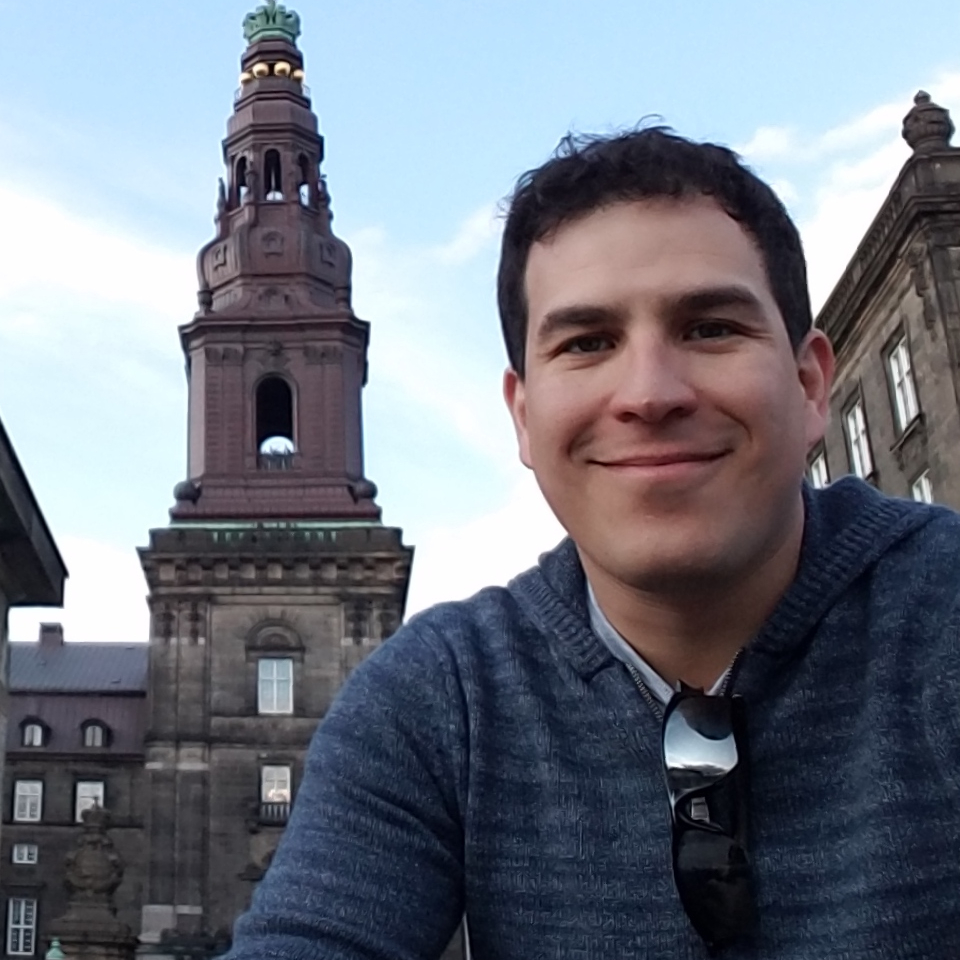}}]{Carlos Nieto-Granda} is a Research
  Scientist in the Science of Intelligent Systems Division at the
  U.S. Army Research Laboratory (DEVCOM/ARL). He has obtained a
  B.S. degree in Electronics Systems from Tecnol\'ogico de Monterrey,
  Campus Estado de Mexico, Mexico; an M.S. degree in Computer Science
  from Georgia Institute of Technology; and a Ph.D. degree in
  Intelligent Systems, Robotics, and Control from University of
  California San Diego. His research interests include autonomous
  exploration, coordination, and decision-making for heterogeneous
  multi-robot teams focused on state estimation, sensor fusion,
  computer vision, localization and mapping, autonomous navigation,
  and control in complex environments. He is a recipient of the 2022
  Transactions on Robotics King-Sun Fu Memorial Best Paper Award.
\end{IEEEbiography}

\begin{IEEEbiography}[{\includegraphics[width=1in,clip,keepaspectratio]{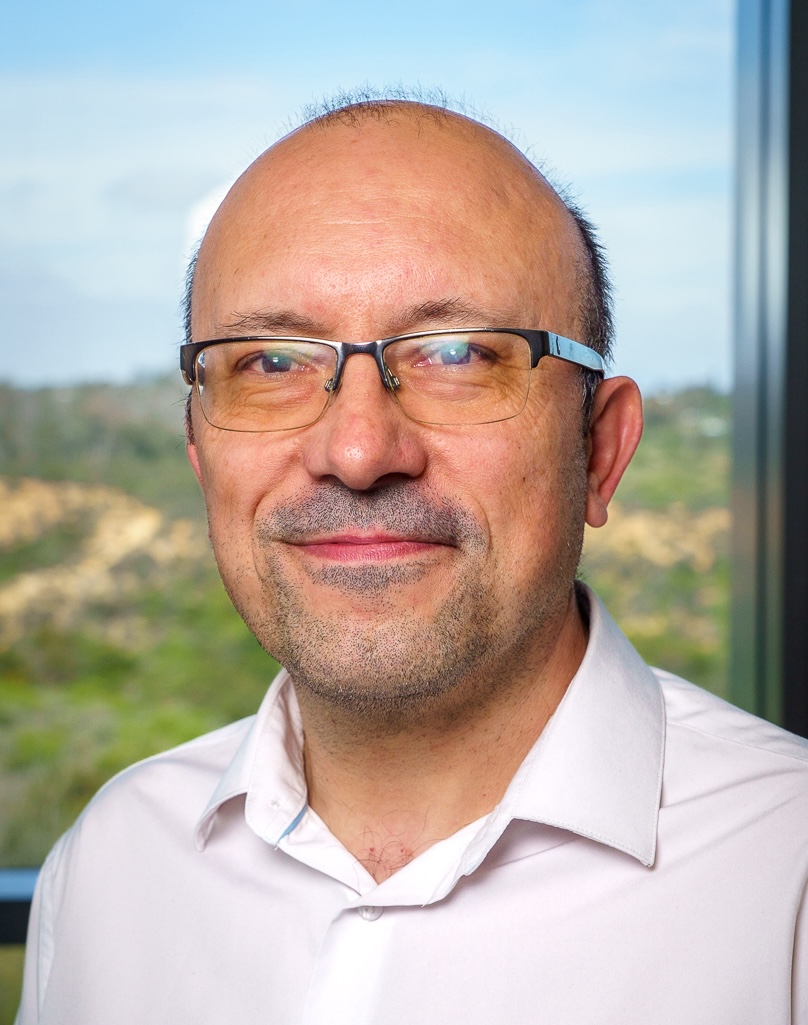}}]{Jorge
    Cort\'{e}s}(M'02, SM'06, F'14) received the Licenciatura degree in
  mathematics from Universidad de Zaragoza, Zaragoza, Spain, in 1997,
  and the Ph.D. degree in engineering mathematics from Universidad
  Carlos III de Madrid, Madrid, Spain, in 2001. He held postdoctoral
  positions with the University of Twente, Twente, The Netherlands,
  and the University of Illinois at Urbana-Champaign, Urbana, IL,
  USA. He was an Assistant Professor with the Department of Applied
  Mathematics and Statistics, University of California, Santa Cruz,
  CA, USA, from 2004 to 2007. He is a Professor and Cymer Corporation
  Endowed Chair in High Performance Dynamic Systems Modeling and
  Control at the Department of Mechanical and Aerospace Engineering,
  University of California, San Diego, CA, USA.  He is a Fellow of
  IEEE, SIAM, and IFAC.  His research interests include distributed
  control and optimization, network science, nonsmooth analysis,
  reasoning and decision making under uncertainty, network
  neuroscience, and multi-agent coordination in robotic, power, and
  transportation networks.
\end{IEEEbiography}

\end{document}